\documentclass[preprint, a4paper]{elsart}
\usepackage{graphics,epsfig}
\usepackage{amsthm}
\usepackage[latin1]{inputenc}
\usepackage[T1]{fontenc}
\usepackage{color}
\usepackage{url}

\usepackage{times, tikz}
\usepackage{todonotes}

\usepackage{amsfonts}
\usepackage{stmaryrd}
\usepackage{microtype}
\usepackage{booktabs}
\usetikzlibrary{shapes}

\usepackage{amscd, amsmath, amssymb, epic, eepic,  delarray}

\newcommand{\EL}{\ensuremath{\mathcal{EL}}}

\newcommand{\ELext}{\ensuremath{\mathcal{ELU}}}
\newcommand{\ALC}{\ensuremath{\mathcal{ALC}}}

\newcommand{\K}{\ensuremath{\mathcal{K}}}

\newcommand{\Cmc}{\ensuremath{\mathcal{C}}}

\newcommand{\Smc}{\ensuremath{\mathcal{S}}}

\usepackage{ifthen}
\newcommand{\isdraft}{\boolean{true}} 

\newcommand{\markupdraft}[2]{
    \ifthenelse{\equal{#1}{display}}{#2}{}
    \ifthenelse{\equal{#1}{color}}{\color{#2}}{}
}

\newcommand{\newcolored}[3][]{{\markupdraft{color}{#2}#3}
    \ifthenelse{\equal{#1}{}}{}{\markupdraft{display}{{\color{yellow!70!black}[#1]}}}} 

\newcommand{\new}[1]{#1}

\ifthenelse{\isdraft}{}{\renewcommand{\markupdraft}[2]{}}

\newcommand{\Isa}[1]{\textcolor{black}{#1}}
\newcommand{\Marc}[1]{\textcolor{black}{#1}}
\newcommand{\Marcb}[1]{\textcolor{black}{#1}}
\newcommand{\Jamal}[1]{\textcolor{black}{#1}}

\journal{Artificial Intelligence}

\newtheorem{definition}{Definition}
\newtheorem{notation}{Notation}

\newtheorem{proposition}{Proposition}
\newtheorem{theorem}{Theorem}
\newtheorem{corollary}{Corollary}
\newtheorem{example}{Example}

\newtheorem{remark}{Remark}

\def\sl#1{\underline{#1}}

\newenvironment{ex.}{
        \medskip
        \noindent {\bf Example}}
        {\hfill$\Box$ \medskip}

\def\cal#1{\ensuremath\mathcal #1}

\def\presuper#1#2%
  {\mathop{}%
   \mathopen{\vphantom{#2}}^{#1}%
   \kern-\scriptspace%
   #2}

\begin{document}

\begin{frontmatter}
\title{Belief Revision, Minimal Change and Relaxation: A General Framework based on \Marc{Satisfaction Systems}\Isa{, and Applications to Description Logics}}

\author[MAS]{Marc Aiguier,}
\author[LAMSADE]{Jamal Atif,}
\author[LTCI]{Isabelle Bloch,}
\author[MAS]{C\'eline Hudelot}
\address[MAS]{MICS, Centrale Supelec, Universit\'e Paris-Saclay, France, \\ {\it \{marc.aiguier,celine.hudelot\}@centralesupelec.fr}}
\address[LAMSADE]{PSL, Universit\'e Paris-Dauphine, LAMSADE, UMR 7243, France, \\ {\it jamal.atif@dauphine.fr}}
\address[LTCI]{LTCI, CNRS, T\'el\'ecom ParisTech, Universit\'e Paris-Saclay, Paris, France, \\ {\it isabelle.bloch@telecom-paristech.fr}}

\maketitle

\begin{abstract}
Belief revision of knowledge bases represented by
a set of sentences in a given logic has been extensively
studied but for specific logics, mainly propositional, and also recently
Horn and description logics. Here, we
propose to generalize this operation from a model-theoretic
point of view, \Marc{by defining revision in an abstract model theory known under the
name of satisfaction systems}. In this framework,
we generalize to \Marc{any satisfaction systems} the characterization
of the well known AGM postulates given by Katsuno and Mendelzon
for propositional logic in terms of minimal change among interpretations.
Moreover, we study how to define revision, satisfying the AGM postulates, from relaxation
notions that have been first introduced in
description logics to define dissimilarity measures
between concepts, and the consequence of which is
to relax the set of models of the old belief until it
becomes consistent with the new pieces of knowledge.
We show how the proposed general framework can be instantiated in different logics such as propositional, \Jamal{first-order}, description and Horn logics. \Isa{In particular for description logics, we introduce several concrete relaxation operators tailored for the description logic $\ALC{}$ and its fragments $\EL{}$ and $\ELext{}$,  discuss their properties and provide some illustrative examples.}
\end{abstract}


\end{frontmatter}



\section{Introduction}

Belief change is an important field of knowledge representation. It is defined by three change operations, {\em expansion, contraction and revision}, that make an agent's belief evolve with \Isa{newly} acquired knowledge.  Belief expansion consists in adding new knowledge without checking consistency, while both contraction and revision consist in consistently removing and adding new knowledge, respectively. When knowledge bases are logical theories, i.e. a set of sentences in a given logic, these changes are governed by a set of postulates \Isa{proposed} for the first time by Alchourr\`on, Gardenfors and Makinson \cite{AGM85}, and since known as the AGM theory. \Marc{Although defined in the abstract framework of logics given by Tarski~\cite{Tar56} (so called Tarskian logics), postulates of the AGM theory make strong assumptions on the considered logics. Indeed, in~\cite{AGM85} the considered logics} have to be closed under the standard propositional connectives in $\{\wedge,\vee,\neg,\Rightarrow\}$, to be compact (i.e. property entailment depends on a finite set of axioms), and to satisfy  the deduction theorem (i.e. entailment and implication are equivalent). While compactness is a standard property of logics, to be closed under the standard propositional connectives is more questionable. Indeed, many non-classical logics such as description logics, equational logic or Horn clause logic,  widely used for various modern applications in computing science, do not satisfy such a constraint. \Isa{Recently, in many works, belief change has been studied in such non-classical logics~\cite{DP15,flouris2005applying,RW14,RWFA13}}. In this direction, we can cite Ribeiro \& al.'s work in \cite{RWFA13} that studies contraction at the abstract level of Tarskian logics, and the recent work in~\cite{Zhuang2015} on the extension of AGM contraction to \Jamal{arbitrary} logics. The adaptation of AGM postulates for revision for non-classical logics has been studied but \Isa{only} for specific logics, mainly  description logics~\Isa{\cite{flouris2006inconsistencies,flouris2005applying,qi2006revision,QLB06,RW09,RW10,WWT10}} and Horn logics~\cite{DP11,ZPZ13}. The reason is that revision can be abstractly defined in terms of expansion and retraction following the Levi identity, but this requires the use of negation, which rules out \Isa{some} non-classical logics~\cite{RW14}. 

In~\cite{KM91} some AGM postulates are interpreted in terms of minimal change, in the sense that the models of the revision should be as close as possible, according to some metric, to the models of the initial knowledge set.
Recently, both for contraction and revision, \Isa{generalizations of the} AGM theory have been proposed in the framework of Tarskian Logics considering minimality criteria on removed formulas~\cite{RW14,RWFA13}. The aim was to study contraction and revision for a larger family of logics containing non-classical ones such as description logics and Horn logics. \Isa{However, to the best of} our knowledge, the generalization of AGM theory with minimality criteria on the set of models of knowledge bases has never been proposed. The reason is that semantics is not explicit in the abstract framework of logics defined by Tarski. 

We propose here to generalize AGM revision but in an abstract model theory, \Marc{satisfaction systems}~\cite{GB85,MDT09}, which formalizes the intuitive notion of logical system, including syntax, semantics and the satisfaction relation. Then, we propose to generalize to any \Marc{satisfaction system} the approach developed in~\cite{KM91} for propositional logic and in \cite{QY08} for description logics. In this abstract framework, we will also show how to define revision operators from the relaxation notion that has been introduced in description logics to define dissimilarity measures between concepts~\cite{DAB14a,DAB14b} and the consequence of which is to relax the set of models of the old belief until it becomes consistent with the new pieces of knowledge. \Jamal{This notion of relaxation, defined in an abstract way through a set of properties, turns out to generalize several revision operators introduced in different contexts e.g.~\cite{cravo2001,meyer2005knowledge,QLB06,Gorogiannis2008a}. This is another key contribution of our work.}. 

\Isa{We provide examples of relaxations in propositional logics, first order logics, \Jamal{and} Horn logic. The case of description logics (DLs) is more detailed, since} DLs  are now pervasive in many knowledge-based representation systems, such as ontological reasoning, semantic web, scene understanding, cognitive robotics, to mention a few. In all these domains, the expert knowledge is not fixed, but rather a flux evolving over time, hence requiring  the definition of rational revision operators. Revision is then a cornerstone in ontology engineering life-cycle where the expert knowledge is prone to change and inconsistency. 
Due to this growing interest in DLs,  several attempts to generalize the well-known AGM theory, making it compliant with the meta-logical flavor of these logics, have been introduced recently, \Isa{as mentioned above.}
The first efforts concentrated on the adaptation of contraction postulates, but more recently, the adaptation of revision postulates and the introduction of new minimality criteria were also addressed~\cite{RW14}, not necessarily related to the contraction  operator, throwing out the need for negation. One can find in~\cite{QLB06} an attempt to adapt the AGM revision postulates to DL in a model-theoretic way, following the seminal work of~\cite{KM91} that translated the AGM postulates in propositional logic semantics. 



\Jamal{To summarize,  our aim is to introduce a general framework for defining easily instantiable concrete revision operators for arbitrary logics. This goes beyond discussing the validity of the AGM theory for some non-classical logics such as description logics or Horn clauses logics which have been a focus of intensive research during the last years, as mentioned above. Indeed, by formulating the AGM theory in the framework of satisfaction systems, we show that one can push the envelop of the AGM theory to make it suitable to some non-classical logics (at the price of loosing or adapting some properties) and define revision operators that can be adapted in quite a straightforward manner to different logics, including non-classical ones. 
Hence, our paper participates in the recent effort for generalizing the AGM theory to non-classical logics. In particular, we introduce a meta-framework that can, by its general and abstract flavor, reduce this effort or at least make it easier. Besides,  we introduce a concrete way of defining revision operators in different logics including non-classical ones, and focus on the particular case of DL, which is of great current interest in semantic web related applications. 
}

The paper is organized as follows. Section~\ref{institutions} reviews some concepts, notations and terminology about \Marc{satisfaction systems} which are used in this work. In Section~\ref{AGM postulates for revision}, we adapt \Isa{the} AGM theory in the framework of \Marc{satisfaction systems}, and then give an abstract model-theoretic rewriting of \Isa{the} AGM postulates. We then show in Section~\ref{Orders and AGM postulates} that any revision operator satisfying such postulates accomplishes an update with minimal change to the set of models of knowledge bases. In Section~\ref{Relaxation and AGM postulates}, we introduce a general framework of relaxation-based revision operators and show that our revision operators lead to faithful assignments and then \Isa{also} satisfy \Isa{the} AGM postulates. In Section~\ref{applications}, we illustrate our abstract approach by providing revision operators in different logics, including classical logics (propositional and first order logics) and non-classical ones (Horn and description logics). \Isa{The case of DL is further developed in Section~\ref{sec:theoryrelaxation}, with several examples.} Finally, Section~\ref{related works} is dedicated to related works.

%
\section{Satisfaction systems}
\label{institutions}


\Marc{Satisfaction systems~\cite{MDT09} (``rooms" in the terminology of~\cite{GB85}) generalize Tarski's classical ``semantic definition of truth''~\cite{Tar44} and Barwise's ``Translation Axiom''~\cite{Bar74}. For the sake of generalization, sentences are simply required to form a set. All other contingencies such as inductive definition of sentences are not considered. Similarly, models are simply seen as elements of a class, \emph{i.e.}~no particular structure is imposed on them.}

\subsection{Definition and examples}
\label{sec:BasicDef}

\begin{definition}[Satisfaction system]
\label{satisfaction system} 

A {\bf satisfaction system} $\mathcal{R}=(Sen,Mod,\models)$ consists of

\begin{itemize}
\item a set $Sen$ of {\bf sentences},
\item a class $Mod$ of {\bf models}, and
\item a satisfaction relation $\models \subseteq (Mod \times Sen)$.
\end{itemize}
\end{definition}

\Marc{Let us note that the non-logical vocabulary, so-called {\em signature}, over which sentences and models are built, is not specified in Definition~\ref{satisfaction system}. Actually, it is left implicit. Hence, as we will see in the examples developed in the paper, a satisfaction system always depends on a signature. \\ There is an extension of satisfaction systems that takes into account explicitly the notion of signatures, the theory of institutions~\cite{GB92}. The theory of institutions is a categorical model theory which has emerged in computing science studies of software specifications and semantics.  Since their introduction, institutions have become a common tool in the area of formal specification mainly to abstractly study the preservation of properties through the structuring of specifications and programs represented by signature morphisms. In this paper, as all the results that we will study about revision will  always be done for logical theories over a same signature, signature morphisms and their interpretation for model classes and sentence sets are not useful. \Isa{This is why we consider the framework of satisfaction systems in this paper.} The advantage is to allow us to abstract from all underlying categorical concepts such as category, functor and other advanced notions such as adjunction, pushout, colimit, etc.} 

\Marc{\begin{example}
\label{examples of institutions}
The following examples of satisfaction systems are of particular importance in computer science and in the remainder of this paper.  
\begin{description}
\item[Propositional Logic (PL)] 
Given a set of propositional variables $\Sigma$, we can define the satisfaction system $\mathcal{R} = (Sen,Mod,\models)$ where $Sen$ is the least set
of sentences finitely built over propositional variables in $\Sigma$ and Boolean connectives in $\{\neg,\vee\}$, $Mod$ contains all the  mappings $\nu:\Sigma\to\{0,1\}$ ($0$ and $1$ are the usual truth values), and the satisfaction relation $\models$ is the usual propositional satisfaction.
\item[Horn Logic (HCL)] A \emph{Horn clause} is a sentence of the form $\Gamma \Rightarrow
\alpha$ where $\Gamma$ is a finite conjunction of propositional variables and $\alpha$
is a propositional variable. The satisfaction system of Horn clause logic is then defined as for {\bf PL} except that sentences are restricted to be conjunctions of Horn clauses.
\item[Many-sorted First Order Logic (FOL)] Signatures are triplets $(S,F,P)$ where $S$ is a set
of sorts, and $F$ and $P$ are sets of function and predicate names respectively, both with arities in $S^\ast\times S$ and $S^+$ respectively ($S^+$ is the set of all non-empty sequences of elements in $S$ and $S^\ast=S^+\cup\{\epsilon\}$ where $\epsilon$ denotes the empty sequence). In the following, to indicate that a function name $f \in F$ (respectively a predicate name $p \in P$) has for arity $(s_1 \ldots s_n,s)$ (respectively $s_1 \ldots s_n$), we will note $f : s_1 \times \ldots \times s_n \to s$ (resp. $p : s_1 \times \ldots \times s_n$). \\ 
Given a signature $\Sigma=(S,F,P)$, we can define the satisfaction system $\mathcal{R} =(Sen,Mod,\models)$ where:
\begin{itemize}
\item $Sen$ is the least set of sentences built over atoms of the form $p(t_1,\ldots,t_n)$ where $p:s_1 \times \ldots \times s_n \in P$ and $t_i \in T_F(X)_{s_i}$ for every $i$, $1 \leq i \leq n$ ($T_F(X)_s$ is the term algebra of sort $s$ built over $F$ with sorted variables in a given set $X$) by finitely applying Boolean connectives in $\{\neg,\vee\}$ and the quantifier $\forall$. 
\item $Mod$ is the class of models $\mathcal{M}$ defined by a family $(M_s)_{s \in S}$ of sets (one for every $s \in S$), each one equipped with a function $f^{\cal M} : M_{s_1} \times \ldots \times M_{s_n} \rightarrow M_s$ for every $f:s_1 \times \ldots \times s_n \rightarrow s  \in F$ and with an n-ary relation $p^{\cal M} \subseteq M_{s_1} \times \ldots \times M_{s_n}$ for every $p:s_1 \times \ldots \times s_n \in P$. 
\item Finally, the satisfaction relation $\models$ is the usual first-order satisfaction. 
\end{itemize}
As for {\bf PL}, we can consider the logic {\bf FHCL} of first-order Horn Logic whose models are those of {\bf FOL} and sentences are restricted to be conjunctions of universally quantified Horn sentences (i.e. sentences of the form $\Gamma \Rightarrow \alpha$ where $\Gamma$ is a finite conjunction of atoms and $\alpha$ is an atom).
\item[Description logic (DL)] Signatures are triplets $(N_C,N_R,I)$ where $N_C$, $N_R$ and $I$ are nonempty pairwise disjoint sets where elements in $N_C$, $N_R$ and $I$ are called concept names, role names and individuals, respectively.  \\ Given a signature $(N_C,N_R,I)$, we can define the satisfaction system $\mathcal{R} =(Sen,Mod,\models)$ where:
\begin{itemize}
\item $Sen$ contains~\footnote{The description logic defined \Isa{here} is better known under the acronym $\mathcal{ALC}$.} all the sentences of the form $C \sqsubseteq D$, $x:C$ and $(x,y):r$ where $x,y \in I$, $r \in N_R$ and $C$ is a concept inductively defined from $N_C \cup \{ \top \}$ and binary and unary operators in $\{\_ \sqcap \_,\_ \sqcup \_\}$ and in $\{\_^c,\forall r.\_,\exists r.\_\}$, respectively. 
\item $Mod$ is the class of models $\mathcal{I}$ defined by a set $\Delta^\mathcal{I}$ equipped for every concept name $A \in N_C$ with a set $A^\mathcal{I} \subseteq \Delta^\mathcal{I}$, for every relation name $r \in N_R$ with a binary relation $r^\mathcal{I} \subseteq \Delta^\mathcal{I} \times \Delta^\mathcal{I}$, and for every individual $x \in I$ with a value $x^\mathcal{I} \in \Delta^\mathcal{I}$. 
\item The satisfaction relation $\models$ is then defined as:
\begin{itemize}
\item $\mathcal{I} \models C \sqsubseteq D$ iff $C^\mathcal{I} \subseteq D^\mathcal{I}$,
\item ${\cal I} \models x:C$ iff $x^{\cal I} \in C^{\cal I}$,
\item ${\cal I} \models (x,y):r$ iff $(x^{\cal I},y^{\cal I}) \in r^{\cal I}$,
\end{itemize}
where $C^\mathcal{I}$ is the evaluation of $C$ in $\mathcal{I}$ inductively defined on the structure of $C$ as follows:
\begin{itemize}
\item if $C = A$ with $A \in N_C$, then $C^{\cal I} = A^{\cal I}$;
\item if $C = \top$ then $C^\mathcal{I} = \Delta^\mathcal{I}$;
\item if $C = C' \sqcup D'$ (resp. $C = C' \sqcap D'$), then $C^{\cal I} = C'^{\cal I} \cup D'^{\cal I}$ (resp. $C^{\cal I} = C'^{\cal I} \cap D'^{\cal I}$);
\item if $C = C'^c$, then $C^{\cal I} = \Delta^{\cal I} \setminus C'^{\cal I}$;
\item if $C = \forall r.C'$, then $C^{\cal I} = \{x \in \Delta^{\cal I} \mid \forall  y \in \Delta^{\cal I}, (x,y) \in r^{\cal I} \mbox{ implies } y \in C'^{\cal I}\}$;
\item if $C = \exists r.C'$, then $C^{\cal I} = \{x \in \Delta^{\cal I} \mid \exists  y \in \Delta^{\cal I}, (x,y) \in r^{\cal I} \mbox{ and } y \in C'^{\cal I}\}$.
\end{itemize}
\end{itemize}
\end{description}
\end{example}}

\subsection{Knowledge base and theories}

Let us now consider a fixed but arbitrary satisfaction system \\ $\mathcal{R} = (Sen,Mod,\models)$. 

\medskip
\begin{notation}
Let $T \subseteq Sen$ be a set of sentences.
\begin{itemize}
\item $Mod(T)$ is the sub-class of $Mod$ whose elements are models of $T$, i.e. for every $\mathcal{M} \in Mod(T)$ and every $\varphi \in T$, $\mathcal{M} \models \varphi$. \Marcb{When $T$ is restricted to a formula $\varphi$ (i.e. $T = \{\varphi\}$), we will denote $Mod(\varphi)$, the class of model of $\{\varphi\}$, rather than $Mod(\{\varphi\})$.}
\item $Cn(T) =\{\varphi \in Sen \mid \forall {\cal M} \in Mod(T),~{\cal M} \models \varphi\}$ is the set of so-called {\em semantic consequences of $T$}.~\footnote{Usually, in the framework of satisfaction systems and institutions, the set of semantic consequences of a theory $T$ is noted $T^\bullet$. Here, we prefer the notation $Cn(T)$ because it will allow us to make a connection with the abstraction of logics as defined by Tarski~\cite{Tar56} and widely used in works dealing with belief change such as revision, expansion or contraction.}
\item Let $\mathbb{M} \subseteq Mod$. Let us note $\mathbb{M}^* = \{\varphi \in Sen \mid \forall \mathcal{M} \in \mathbb{M}, \mathcal{M} \models \varphi\}$. Therefore, we have for every $T \subseteq Sen$, $Cn(T) = Mod(T)^*$. When $\mathbb{M}$ is restricted to one model $\mathcal{M}$, $\mathbb{M}^*$ will be equivalently noted $\mathcal{M}^*$.
\item Let us note $Triv = \{\mathcal{M} \in Mod \mid \mathcal{M}^* = Sen\}$.  
\end{itemize}
\end{notation}
\Marcb{Let us note that for every $T \subseteq Sen$, $Triv \subseteq Mod(T)$.}

From the above notations, we obviously have: 
\begin{equation}
Cn(T) = Cn(T') \Leftrightarrow Mod(T) = Mod(T').
\label{eq:equivCnMod}
\end{equation}

\Isa{The two functions} $Mod(\_)$ and $\_^*$ form what is known as a Galois connection in that they satisfy the following properties: for all $T,T' \subseteq Sen$ and $\mathbb{M},\mathbb{M}' \subseteq Mod$, we have (see~\cite{Dia08})
\begin{enumerate}
\item $T \subseteq T' \Longrightarrow Mod(T') \subseteq Mod(T)$
\item $\mathbb{M} \subseteq \mathbb{M}' \Longrightarrow {\mathbb{M}'}^* \subseteq \mathbb{M}^*$
\item $T \subseteq Mod(T)^*$
\item $\mathbb{M} \subseteq Mod(\mathbb{M}^*)$
\end{enumerate}

\medskip
\begin{definition}[Knowledge base and theory]
\label{theory}
A {\bf knowledge base} $T$ is a set of sentences (i.e. $T \subseteq Sen$).
A knowledge base $T$ is said to be a {\bf theory} if and only if $T = Cn(T)$. \\ A theory $T$ is {\bf finitely representable} if there exists a finite set $T' \subseteq Sen$  such that $T = Cn(T')$.
\end{definition}

\medskip
\begin{proposition}
For every satisfaction system $\mathcal{R}$, we have:

\begin{description}
\item[Inclusion] $\forall T \subseteq Sen, T \subseteq Cn(T)$;
\item[Iteration] $\forall T \subseteq Sen, Cn(T) = Cn(Cn(T))$;
\item[Monotonicity] $\forall T,T' \subseteq Sen, T \subseteq T' \Longrightarrow Cn(T) \subseteq Cn(T')$. 
\end{description}
\end{proposition}

\begin{proof}
Inclusion and iteration are obvious properties of the mapping $Cn$ by definition (inclusion is Property 3 of the Galois Connection above). \\ Suppose $T \subseteq T'$. By the first property of the Galois connection above,
we have that $Mod(T') \subseteq Mod(T)$
and $Mod(T)^* \subseteq Mod(T')^*$ from Property 2, hence $Cn(T) \subseteq Cn(T')$.
\end{proof}

Hence, satisfaction systems are {\em Tarskian} according to the definition of logics given by Tarski under which a logic is a pair $(\mathcal{L},Cn)$ where $\mathcal{L}$ is a set of expressions (formulas) and $Cn : \mathcal{P}(\mathcal{L}) \to \mathcal{P}(\mathcal{L})$ is a mapping that satisfies the inclusion, iteration and monotonicity properties~\cite{Tar56}. \Marc{Indeed, from any satisfaction system $\mathcal{R}$ we can define the following Tarskian logic $(\mathcal{L},Cn)$ where $\mathcal{L} = Sen$ and $Cn$ is the mapping that associates to every $T \subseteq Sen$, the set $Cn(T)$ of semantic consequences of $T$.}

\medskip
Classically, \Isa{the} consistency of a theory $T$ is defined as $Mod(T) \neq \emptyset$. The problem of such a definition of consistency is that its significance depends on the actual logic. Hence, this consistency is significant for {\bf FOL}, while in {\bf FHCL} it is a trivial property since each set of sentences is consistent because $Mod(T)$ always contains the trivial model. Here, \Isa{for the consistency notion be} more appropriate with our purpose to define revision for the largest family of logics, we propose a more general definition of consistency, the meaning of which is that there is at least a sentence which is not a semantic consequence. 

\medskip
\begin{definition}[Consistency]
$T \subseteq Sen$ is {\bf consistent} if $Cn(T) \neq Sen$.
\end{definition}

\medskip
\begin{proposition}
\label{proposition:consistency}
For every $T \subseteq Sen$, $T$ is consistent if and only if $Mod(T) \setminus Triv \neq \emptyset$.
\end{proposition}

\begin{proof} Let us prove that $Cn(T) = Sen$ iff $Mod(T) \setminus Triv = \emptyset$.
Let us first assume that $Mod(T) \setminus Triv = \emptyset$. Therefore, by definition of $Cn(T)$, this means that the only models that satisfy $T$ are $\mathcal{M}$ such that $\mathcal{M}^* = Sen$ (if they exist). Hence, we have $Cn(T) = Sen$. \\
Conversely, let us assume that $Cn(T) = Sen$. This means that every model $\mathcal{M}$ such that $\mathcal{M}^* \neq Sen$ does not belong to $Mod(T)$, and $Mod(T)\setminus Triv = \emptyset$. 
\end{proof}

\begin{corollary}
\label{cor:triv}
For every $T \subseteq Sen$, $T$ is inconsistent is equivalent to $Mod(T) = Triv$.
\end{corollary}







\section{AGM postulates for revision in \Marc{satisfaction systems}}
\label{AGM postulates for revision}

\subsection{AGM postulates}
\label{AGM postulates}

AGM postulates for knowledge base revision in satisfaction systems are easily adaptable. \Jamal{We build upon the model-theoretic characterization introduced by Katsuno and Mendelzon (KM)~\cite{KM91} for propositional logic. Note, however, that in propositional logic, a belief base can be represented by a formula, and then KM postulates exploit this property. This is no more the case in our context, but we argue that the postulates are still appropriate.} Given two knowledge bases $T, T' \subseteq Sen$, $T \circ T'$ denotes the {\bf revision of $T$ by $T'$}, that is,  $T \circ T'$ is obtained by adding consistently new knowledge $T'$ to the old knowledge base $T$. Note that $T \circ T'$ cannot be defined as $T \cup T'$ because nothing ensures that $T \cup T'$ is consistent. The revision operator has then to change minimally $T$ so that $T \circ T'$ is consistent. This is what the AGM postulates ensure.

\begin{description}
\item[(G1)] If $T'$ is consistent, then so is $T \circ T'$.
\item[(G2)] $Mod(T \circ T') \subseteq Mod(T')$.
\item[(G3)] if $T \cup T'$ is consistent, then $T \circ T' = T \cup T'$.
\item[(G4)] if $Cn(T'_1) = Cn(T'_2)$, then $Mod(T \circ T'_1) = Mod(T \circ T'_2)$.
\item[(G5)] $Mod((T \circ T') \cup T'') \subseteq Mod(T \circ (T' \cup T''))$. 
\item[(G6)] if $(T \circ T') \cup T''$ is consistent, then $Mod(T \circ (T' \cup T'')) \subseteq  Mod((T \circ T') \cup T'')$. 
\end{description}

In the literature such as in~\cite{KM91,QY08}, the following stronger version of Postulate (G4) can be found:

\begin{center}
{\bf (G'4)} if $Cn(T_1) = Cn(T'_1)$ and $Cn(T_2) = Cn(T'_2)$, then $Mod(T_1 \circ T_2) = Mod(T'_1 \circ T'_2)$
\end{center}

\medskip
\begin{remark}
This stronger version of Postulate (G4) states a complete independence of the syntactical forms of both the original knowledge base and the newly \Isa{acquired knowledge}. The problem with Postulate (G'4) is that it is almost never satisfied when we want to preserve the structure of knowledge bases and then apply revision operators over the formulas that compose knowledge bases. Indeed, let us consider in the logic {\bf PL} the following knowledge bases $T_1 = \{p,q\}$ and $T_2 = \{q \Rightarrow p,q\}$ over the signature $\{p,q\}$. Obviously, we have that $Mod(T_1) = Mod(T_2) = \{\nu : p \mapsto 1,q \mapsto 1\}$. Let us consider the knowledge base $T' = \{\neg q\}$. We have now that $T_1 \cup T'$ (and then $T_2 \cup T'$) is inconsistent. A way to retrieve the consistency is to replace in $T_1$ and $T_2$ the atomic formula $q$ by $\neg q$.
Hence, $T_1 \circ T' = \{p,\neg q\}$ and $T_2 \circ T' = \{q \Rightarrow p,\neg q\}$. Then $Mod(T_1 \circ T') = \{\nu : p \mapsto 1,q \mapsto 0\}$,  $Mod(T_2 \circ T') = \{\nu : p \mapsto 1,q \mapsto 0;\nu' : p \mapsto 0,q \mapsto 0\}$, and $Mod(T_1 \circ T') \neq Mod(T_2 \circ T')$. \\ 
In~\cite{KM91}, the authors bypass the problem by representing any knowledge base $K$ (which is a theory in~\cite{KM91}) by a propositional formula $\psi$ such that $K = Cn(\psi)$. Hence, they apply their revision operator on $\psi$ and not on $K$, and so they lose the structure of the knowledge base $K$. 
\end{remark}

\Marc{Another remarkable point to note is that now Postulate (G4)  in this weaker form can be derived from the other postulates.}

\Marc{\begin{proposition}
\label{proposition:reduced set of postulates}
Postulates (G1)-(G3), (G5) and (G6) imply Postulate (G4).
\end{proposition}}
\Marc{\begin{proof}
Let us suppose that $Cn(T'_1) = Cn(T'_2)$. Here, three cases have to be considered:
\begin{enumerate}
\item \Isa{One of $T'_1$ and $T'_2$ is inconsistent (say $T'_1$ without loss of generality).} Since $Cn(T'_1) = Cn(T'_2)$ by hypothesis, \Isa{$T'_2$}  is also inconsistent. By Postulate (G2), we then have that, \Isa{for $i=1,2$, $Mod(T \circ T'_i) \subseteq Mod(T_i)$, and $Mod(T_i)= Triv$ (Corollary~\ref{cor:triv}). Hence $Mod(T \circ T'_i) \subseteq Triv$, and $Mod(T \circ T'_1) = Mod(T \circ T'_2) = Triv$.}
\item \Isa{Both $T \cup T'_1$ and $T \cup T'_1$ are consistent}. Since $Cn(T'_1) = Cn(T'_2)$, we know that $Mod(T'_1) = Mod(T'_2)$ \Isa{(Equation~\ref{eq:equivCnMod})}, and then $Mod(T \cup T'_1) = Mod(T \cup T'_2)$. Therefore, by Postulate (G3), we have that $Mod(T \circ T'_1) = Mod(T \circ T'_2)$. 
\item $T'_1$ and $T'_2$ are consistent but \Isa{$T \cup T'_1$ or $T \cup T'_2$} is not \Isa{(say $T \cup T'_1$)}. From $Cn(T'_1) = Cn(T'_2)$, \Isa{we derive that $T \cup T'_2$} is also inconsistent. By Postulate (G1), both $T \circ T'_1$ and $T \circ T'_2$ are consistent. Let $\mathcal{M} \in Mod(T \circ T'_1)$. If $\mathcal{M} \in Triv$, then obviously $\mathcal{M} \in Mod(T \circ T'_2)$. Therefore, let us suppose that $\mathcal{M} \not\in Triv$. By Postulate (G2), $\mathcal{M} \in Mod(T'_1)$, and then $\mathcal{M} \in Mod(T'_2)$. Let $\mathcal{M}' \in Mod(T \circ T'_2) \setminus Triv$. Such a model exists as $T \circ T'_2$ is consistent. By Postulate (G2) and the hypothesis that $Cn(T'_1) = Cn(T'_2)$, $\{\mathcal{M},\mathcal{M}'\}^*$ contains both $T'_1$ and $T'_2$. Obviously, we have that $(T \circ T'_1) \cup  \{\mathcal{M},\mathcal{M}'\}^*$ and $(T \circ T'_2) \cup  \{\mathcal{M},\mathcal{M}'\}^*$ are consistent. Therefore, By Postulates (G5) and (G6), we have that $Mod((T \circ T'_1) \cup  \{\mathcal{M},\mathcal{M}'\}^*) = Mod((T \circ (T'_1 \cup  \{\mathcal{M},\mathcal{M}'\}^*) = Mod(T \circ  \{\mathcal{M},\mathcal{M}'\}^*)$ and $Mod((T \circ T'_2) \cup  \{\mathcal{M},\mathcal{M}'\}^*) = Mod((T \circ (T'_2 \cup  \{\mathcal{M},\mathcal{M}'\}^*) = Mod(T \circ  \{\mathcal{M},\mathcal{M}'\}^*)$. We can then derive that  $Mod((T \circ T'_1) \cup  \{\mathcal{M},\mathcal{M}'\}^*) = Mod((T \circ T'_2) \cup  \{\mathcal{M},\mathcal{M}'\}^*)$, and conclude that $\mathcal{M} \in Mod(T \circ T'_2)$. \Isa{Similarly, by reversing the roles of $T'_1$ and $T'_2$, if $\mathcal{M} \in Mod(T \circ T'_2)$, we can conclude that $\mathcal{M} \in Mod(T \circ T'_1)$.}
\end{enumerate}
\end{proof}}

Intuitively, any revision operator $\circ$ satisfying the six postulates above induces minimal change, that is the models of $T \circ T'$ are the models of $T$ that are the closest to models of \Isa{$T'$,} according to some distance for measuring 
how close are models.
This is what will be shown in the next section by establishing a correspondence between AGM postulates and binary relations over models with minimality conditions.

\subsection{Faithful assignment and AGM postulates}
\label{Orders and AGM postulates}

Let $\mathbb{M} \subseteq Mod$. Let $\preceq$ be a binary relation over $\mathbb{M}$. We define $\prec$ as $\mathcal{M} \prec \mathcal{M}'$ if and only if $\mathcal{M} \preceq \mathcal{M}'$ and $\mathcal{M}' {\not \preceq} \mathcal{M}$. We define $Min(\mathbb{M},\preceq) = \{\mathcal{M} \in \mathbb{M} \mid \forall \mathcal{M}' \in \mathbb{M}, \mathcal{M}' {\not \prec} \mathcal{M}\}$.  

\medskip
\begin{definition}[Faithful assignment]
\label{faithful assignment}
An {\bf assignment} is a mapping that assigns to each knowledge base $T$ a binary relation $\preceq_T$ over $Mod$. We say that this assignment is {\bf faithful (FA)} if the following two  conditions are satisfied:
\begin{enumerate}
\item if $\mathcal{M},\mathcal{M}' \in Mod(T)$, $\mathcal{M} {\not \prec}_T \mathcal{M}'$.
\item for every $\mathcal{M} \in Mod(T)$ and every $\mathcal{M}' \in Mod \setminus Mod(T)$, $\mathcal{M} \prec_T \mathcal{M}'$.
\end{enumerate} 
A binary relation $\preceq_T$ assigned to a knowledge base $T$ by a faithful assignment will be also said {\bf faithful}.
\end{definition}

\medskip
\begin{remark}
The definition of FA differs from the one originally given in~\cite{KM91} on two points:
\begin{enumerate}
\item In~\cite{KM91}, a third condition is stated:
$$\forall T,T' \subseteq Sen, Mod(T) = Mod(T') \Rightarrow \preceq_T = \preceq_{T'}.$$
As for (G'4), this condition expresses a syntactical independence.



\item It is no longer required for $\preceq_T$ to be a pre-order. As shown below, the only important feature to have to make a correspondence between a FA and the fact that $\circ$ satisfies Postulates (G1)-(G6) is that there is a minimal model for $\preceq_T$ in $Mod(T')$ as expressed by Theorem~\ref{correspondence agm and orders}. 
\end{enumerate}
\end{remark}

\medskip
\begin{theorem}
\label{correspondence agm and orders}
Let $\circ$ be a revision operator. $\circ$ satisfies AGM Postulates if and only if there exists a FA that maps each knowledge base $T \subseteq Sen$  to a binary relation $\preceq_T$ such that for every knowledge base $T' \subseteq Sen$:
\begin{itemize}
\item $Mod(T \circ T') \setminus Triv  = Min(Mod(T') \setminus Triv,\preceq_T)$;
\item if $T'$ is consistent, then $Min(Mod(T') \setminus Triv,\preceq_T) \neq \emptyset$;
\item for every $T'' \subseteq Sen$, if $(T \circ T') \cup T''$ is consistent, then $Min(Mod(T') \setminus Triv,\preceq_T) \cap Mod(T'') = Min(Mod(T' \cup T'') \setminus Triv,\preceq_T)$.
\end{itemize}
\end{theorem}


Note that if $T'$ is inconsistent, then so is $T\circ T'$, and we can set arbitrarily $T \circ T' = T'$, which corresponds to a cautious revision. The case where $T$ is inconsistent is not considered in this paper, since in that case \Isa{other operators could be more relevant than revision, in particular debugging methods (see e.g.~\cite{schlobach2003non} for debugging of terminologies, or~\cite{ribeiro2009} for base revision for ontology debugging, both in description logics.}

\begin{proof}

\begin{enumerate}
\item Let us suppose that $\circ$ satisfies AGM Postulates. For every knowledge base $T$, let us define the binary relation $\preceq_T \subseteq Mod \times Mod$ by: for all $\mathcal{M},\mathcal{M}' \in Mod$,

\begin{center}
$\mathcal{M} \preceq_T \mathcal{M}'$ iff 
$
\left\{
\begin{array}{l}
\mbox{either $\mathcal{M} \in Mod(T)$} \\
\mbox{or $\mathcal{M} \in Mod(T \circ \{\mathcal{M},\mathcal{M}'\}^*)$ and $\mathcal{M}' {\not \in} Triv$}
\end{array}
\right.$
\end{center}

Let us first show that $\preceq_T$ satisfies the two conditions of FA. 
\begin{itemize}
\item The first condition easily follows from the definition of $\preceq_T$.
\item To prove the second one, let us assume that $\mathcal{M} \in Mod(T)$ and $\mathcal{M}' \not\in Mod(T)$. Since $\mathcal{M} \in Mod(T)$, we have $\mathcal{M} \preceq_T \mathcal{M'}$. Here two cases have to be considered:

\begin{enumerate}
\item $\mathcal{M} \in Triv$. In this case, we directly have by definition that $\mathcal{M}' \not\preceq_T \mathcal{M}$.
\item $\mathcal{M} \not\in Triv$. Then $T \cup \{\mathcal{M},\mathcal{M}'\}^*$ is consistent since $\mathcal{M} \in Mod(T) \setminus Triv$ and $\mathcal{M} \in Mod(\mathcal{M}^*) \subseteq Mod(\{\mathcal{M},\mathcal{M}'\}^*)$. Then by Postulate (G3), we have that $T \circ \{\mathcal{M},\mathcal{M}'\}^* = T \cup \{\mathcal{M},\mathcal{M}'\}^*$. Therefore, we have that 
$\mathcal{M}' {\not \in} Mod(T \circ \{\mathcal{M},\mathcal{M}'\}^*)$, and $\mathcal{M}' \not\preceq_T \mathcal{M}$. 
\end{enumerate}
Hence $\mathcal{M} \prec_T \mathcal{M'}$ in both cases.
\end{itemize}

Let us now prove the three supplementary conditions. 
\begin{itemize}
\item First, let us show that $Mod(T \circ T') = Min(Mod(T') \setminus Triv,\preceq_T)$. If $T'$ is inconsistent, then by Proposition~\ref{proposition:consistency} $Mod(T') \setminus Triv = \emptyset$, and by (G2) \Isa{$Mod(T \circ T') \subseteq Mod (T') \subseteq Triv$,}
hence $Mod(T \circ T') \setminus Triv = \emptyset = Min(Mod(T')\setminus Triv,\preceq_T)$.\\ 
Let us assume now that $T'$ is consistent.
\begin{itemize}
\item Let us first show that $Mod(T \circ T') \setminus Triv \subseteq Min(Mod(T') \setminus Triv,\preceq_T)$. Let $\mathcal{M} \in Mod(T \circ T') \setminus Triv$. Let us assume that $\mathcal{M} \not\in Min(Mod(T') \setminus Triv,\preceq_T)$. By (G2), $\mathcal{M} \in Mod(T') \setminus Triv$. By hypothesis, there exists $\mathcal{M}' \in Mod(T')\setminus Triv$ such that $\mathcal{M}' \prec_T \mathcal{M}$. Here, two cases have to be considered:

\begin{enumerate}
\item $\mathcal{M}' \in Mod(T)$. As $\mathcal{M}' \in Mod(T') \setminus Triv$, then $T \cup T'$ is consistent, and then by (G3), $T \circ T' = T \cup T'$. Thus, $\mathcal{M} \in Mod(T)$, and then $\mathcal{M} \preceq_T \mathcal{M}'$, which is a contradiction. 
\item $\mathcal{M}' \not\in Mod(T)$. By definition of $\preceq_T$, this means that $\mathcal{M}' \in Mod(T \circ \{\mathcal{M},\mathcal{M}'\}^*)$. As $\mathcal{M},\mathcal{M}' \in Mod(T')$, \Marc{by Postulate (G2),} $(T \circ T') \cup \{\mathcal{M},\mathcal{M}'\}^*$ is consistent, and then by Postulates (G5) and (G6), we have that $\Marc{Mod(}T \circ \{\mathcal{M},\mathcal{M}'\}^*\Marc{)} = \Marc{Mod(}(T \circ T') \cup \{\mathcal{M},\mathcal{M}'\}^*\Marc{)}$. By the hypothesis that $\mathcal{M}' \prec_T \mathcal{M}$, we can deduce that $\mathcal{M} \not\in Mod(T \circ \{\mathcal{M},\mathcal{M}'\}^*)$, whence \Marc{by Postulate (G6)} we have that $\mathcal{M} \not\in Mod(T \circ T') \setminus Triv$, which is a contradiction. 
\end{enumerate}
Finally we can conclude that $\mathcal{M} \in Min(Mod(T') \setminus Triv,\preceq_T)$, and then $Mod(T \circ T') \setminus Triv \subseteq Min(Mod(T') \setminus Triv,\preceq_T)$.

\item Let us now show that $Min(Mod(T') \setminus Triv,\preceq_T) \subseteq Mod(T \circ T') \setminus Triv$. Let $\mathcal{M} \in Min(Mod(T') \setminus Triv,\preceq_T)$. Let us assume that $\mathcal{M} \not\in Mod(T \circ T') \setminus Triv$. As $T'$ is consistent, by Postulates (G1) and (G2), there exists $\mathcal{M}' \in Mod(T \circ T')$ such that ${\mathcal{M}'}^* \neq Sen$, and $\mathcal{M}' \in Mod(T')$. Since $T' \subseteq \{\mathcal{M},\mathcal{M}'\}^*$, we also have that $Mod(T' \cup \{\mathcal{M},\mathcal{M}'\}^*) = Mod(\{\mathcal{M},\mathcal{M}'\}^*)$. By Postulates (G5) and (G6), we can write $Mod(T \circ T') \cap Mod(\{\mathcal{M},\mathcal{M}'\}^*) = Mod(T \circ \{\mathcal{M},\mathcal{M}'\}^*)$, since $(T \circ T') \cup \{\mathcal{M},\mathcal{M}'\}^*$ is consistent. Hence, $\mathcal{M} \not\in Mod(T \circ \{\mathcal{M},\mathcal{M}'\}^*)$, and then $\mathcal{M}' \prec_T \mathcal{M}$, which is a contradiction. We can conclude that $\mathcal{M} \in Mod(T\circ T')\setminus Triv$, and then $Min(Mod(T') \setminus Triv,\preceq_T) \subseteq Mod(T \circ T') \setminus Triv$. 
\end{itemize}
\item Secondly, let us show that $Min(Mod(T') \setminus Triv,\preceq_T) \neq \emptyset$ if $T'$ is consistent. By Postulate (G1), we have that $T \circ T'$ is consistent, and then $Mod(T \circ T') \setminus Triv \neq \emptyset$. We can directly conclude by the previous point that $Min(Mod(T') \setminus Triv,\preceq_T) \neq \emptyset$. 

\item Finally, let us show that for every $T',T'' \subseteq Sen$, $Min(Mod(T') \setminus Triv,\preceq_T) \cap Mod(T'') = Min(Mod(T' \cup T'') \setminus Triv,\preceq_T)$ if $(T \circ T') \cup T''$ is consistent. By (G5) and (G6), we have that $Mod(T \circ (T' \cup T'')) = Mod((T \circ T') \cup T'')$. Therefore, by the first point, we can directly conclude that  $Min(Mod(T') \setminus Triv,\preceq_T) \cap Mod(T'') = Min(Mod(T' \cup T'') \setminus Triv,\preceq_T)$.
\end{itemize}

\item Let us now suppose that \Isa{for a revision operation $\circ$} there exists a FA which maps any knowledge base $T \subseteq Sen$ to a binary relation $\preceq_T \subseteq Mod \times Mod$ satisfying the three conditions of Theorem~\ref{correspondence agm and orders}. Let us prove that $\circ$ verifies the AGM Postulates.
\begin{description}
\item[(G1)] This postulate directly results from the fact that $Min(Mod(T') \setminus Triv,\preceq_T) \neq \emptyset$ when $T'$ is consistent, \Isa{hence $Mod(T\circ T') \setminus Triv \neq \emptyset$}. 
\item[(G2)] Let $\mathcal{M} \in Mod(T \circ T')$. \Isa{If $\mathcal{M} \in Triv$, then obviously $\mathcal{M} \in Mod(T')$. Now, if $\mathcal{M} \notin Triv$, then} by definition, $\mathcal{M} \in Min(Mod(T') \setminus Triv,\preceq_T)$. This means that $\mathcal{M} \in Mod(T')$. 
\item[(G3)] Suppose that $T \cup T'$ is consistent \Isa{(hence $Mod(T \cup T') \setminus Triv \neq \emptyset$)}. 
\begin{itemize}
\item Let us first prove that $Mod(T \circ T') \subseteq Mod(T \cup T')$. Let $\mathcal{M} \in Mod(T \circ T')$. \Marc{Here two cases have to be considered:}
\begin{enumerate}
\item \Marc{$\mathcal{M} \in Triv$. In this case, we obviously have that $\mathcal{M} \in Mod(T \cup T')$.} 
\item \Marc{$\mathcal{M} \not\in Triv$.} By definition, $\mathcal{M} \in Min(Mod(T') \setminus Triv,\preceq_T)$. Hence, we have that $\mathcal{M} \in Mod(T')$. Let us suppose now that $\mathcal{M} {\not \in} Mod(T)$. As $T$ is consistent, $Mod(T)  \setminus Triv \neq \emptyset$ by Proposition~\ref{proposition:consistency}. Therefore, \Marc{there exists} $\mathcal{M}' \in Mod(T) \setminus Triv$ \Marc{such that} $\mathcal{M}' \prec_T \mathcal{M}$ (from $\mathcal{M} \notin Mod(T)$ and the \Marc{second} property of FA), which is a contradiction. \Isa{Hence $\mathcal{M} \in Mod(T)$ and $\mathcal{M} \in Mod (T \cup T')$.}
\end{enumerate}
\item Let us now prove that $Mod(T \cup T') \subseteq Mod(T \circ T')$. Let $\mathcal{M} \in Mod(T \cup T')$ such that $\mathcal{M} {\not \in} Mod(T \circ T')$. Therefore, $\mathcal{M} \in Mod(T)$. By hypothesis, there exists $\mathcal{M}' \in Mod(T') \setminus Triv$ such that $\mathcal{M}' \prec_T \mathcal{M}$ (since $\mathcal{M} \notin Min(Mod(T') \setminus Triv,\preceq_T)$), and then $\mathcal{M}' {\not \in} Mod(T)$ by the first condition of FA. However, by the second condition of FA, we have that $\mathcal{M} \prec_T \mathcal{M}'$, which is a contradiction. 
\end{itemize} 
Finally, we can conclude that $Mod(T \circ T') = Mod(T \cup T')$.
\item[(G5)] Let $\mathcal{M} \in Mod(T \circ T') \cap Mod(T'')$. Let us assume that $\mathcal{M} {\not \in} Min(Mod(T' \cup T'') \setminus Triv,\preceq_T)$. This means that \Marc{$\mathcal{M} \in Triv$ or} there exists $\mathcal{M}' \in Mod(T' \cup T'')$ such that ${\mathcal{M}'}^* \neq Sen$ and $\mathcal{M}' \prec_T \mathcal{M}$. \Marc{In the first case, we obviously have that $\mathcal{M} \in Mod(T \circ (T' \cup T''))$. In the second case,} we then have that $\mathcal{M}' \in Mod(T')$, and then $\mathcal{M}' {\not \prec_T} \mathcal{M}$ since $\mathcal{M} \in Min(Mod(T') \setminus Triv, \preceq_T)$, which is a contradiction. 
\item[(G6)] Let us suppose that $(T \circ T') \cup T''$ is consistent. Let $\mathcal{M} \in Mod(T \circ (T' \cup T''))$. By hypothesis, \Marc{either $\mathcal{M} \in Triv$ and in this case, obviously we have that $\mathcal{M} \in Mod((T \circ T') \cup T'')$, or} $\mathcal{M} \in Min(Mod(T' \cup  T'') \setminus Triv,\preceq_T)$ as $Mod(T \circ (T' \cup T'')) \Marc{\setminus Triv} = Min(Mod(T' \cup  T'') \setminus Triv,\preceq_T)$. As $(T \circ T') \cup T''$ is consistent, we have that $Min(Mod(T' \cup  T'') \setminus Triv,\preceq_T) = Min(Mod(T') \setminus Triv,\preceq_T) \cap Mod(T'')$ and then $\mathcal{M} \in Mod((T \circ T') \cup T'')$.
\end{description}
\end{enumerate}
\end{proof}

\Marc{Given a revision operator $\circ$ satisfying the AGM postulates, any FA satisfying the supplementary conditions of Theorem~\ref{correspondence agm and orders} will be called FA+. To a revision operator $\circ$ satisfying the AGM postulates, we can associate many FA+. \Isa{An example of such a FA+ is}
the mapping $f$ that associates to every $T \subseteq Sen$ the binary relation $\preceq_T$ defined as follows:}

\Marc{Given $T' \subseteq Sen$, let us start by defining $\preceq^{T'}_T \subseteq Mod(T') \times Mod(T')$ as:
$$\mathcal{M} \preceq^{T'}_T \mathcal{M}' \Longleftrightarrow \mathcal{M} \in Mod(T \circ T')~\mbox{and}~\mathcal{M}' {\not \in} Mod(T \circ T').$$
Let us then set $\Isa{f(T) =} \preceq_T = \bigcup_{T'}\preceq^{T'}_T$  (i.e. $\mathcal{M} \preceq_T \mathcal{M}' \Leftrightarrow \exists T', \mathcal{M} \preceq^{T'}_T \mathcal{M}'$).}

\Marc{
\begin{theorem}
If $\circ$ satisfies the AGM postulates, then the mapping $f$ \Isa{defined above} is a FA+.
\end{theorem}
\begin{proof}
First, let us show that $f$ is a FA. 
\begin{itemize}
\item Let $\mathcal{M},\mathcal{M}' \in Mod(T)$. Let us suppose that $\mathcal{M} \prec_T \mathcal{M}'$. This means that there exists $T' \subseteq Sen$ such that $\mathcal{M},\mathcal{M}' \in Mod(T')$, $\mathcal{M} \in Mod(T \circ T')$ and $\mathcal{M}' \not\in Mod(T \circ T')$. Hence we have that $T \cup T'$ is consistent, and then by Postulate (G3), $T \circ T' = T \cup T'$. We then have that $\mathcal{M}' \in Mod(T \circ T')$ which is a contradiction. 
\item Let $\mathcal{M} \in Mod(T)$ and let $\mathcal{M}' \in Mod \setminus Mod(T)$. We have that $\mathcal{M} \preceq^\emptyset_T \mathcal{M}'$, and then $\mathcal{M} \preceq_T \mathcal{M}'$ by definition of $\preceq_T$. Now, let us suppose that $\mathcal{M}' \preceq_T \mathcal{M}$.  This means that there exists $T' \subseteq Sen$ such that $\mathcal{M},\mathcal{M}' \in Mod(T')$, $\mathcal{M}' \in Mod(T \circ T')$ and $\mathcal{M} \not\in Mod(T \circ T')$. But, as $\mathcal{M} \in Mod(T)$, we have that $T \cup T'$ is consistent, and then by Postulate (G3), $T \circ T' = T \cup T'$. Hence, we have that $\mathcal{M} \in Mod(T \circ T')$ which is a contradiction. 
\end{itemize}
Let us show now the supplementary conditions of Theorem~\ref{correspondence agm and orders}. 
\begin{itemize}
\item First, let us show that $Mod(T \circ T') \setminus Triv = Min(Mod(T') \setminus Triv,\preceq_T)$. The case where $T'$ is inconsistent follows the same proof as in Theorem~\ref{correspondence agm and orders}. \\ 
Let us suppose that $T'$ is consistent. Let $\mathcal{M} \in Mod(T \circ T') \setminus Triv$. Let us suppose that $\mathcal{M} \not\in Min(Mod(T') \setminus Triv,\preceq_T)$. This means that there exists $\mathcal{M}' \in Mod(T') \setminus Triv$ such that $\mathcal{M}' \prec_T \mathcal{M}$. Therefore, there exists $T'' \subseteq Sen$ such that $\mathcal{M},\mathcal{M}' \in Mod(T'')$, $\mathcal{M}' \in Mod(T \circ T'')$ and $\mathcal{M} \not\in Mod(T \circ T'')$. Hence, both $(T \circ T') \cup T''$ and $(T \circ T'') \cup T'$ are consistent, and then by Postulates (G5) and (G6), $Mod((T \circ T') \cup T'') = Mod((T \circ T'') \cup T') = Mod(T \circ (T' \cup T''))$. We can then derive that $\mathcal{M} \in Mod(T \circ T'')$ which is a contradiction. \\ 
Let $\mathcal{M} \in Min(Mod(T') \setminus Triv,\preceq_T)$. Let us suppose that $\mathcal{M} \not\in Mod(T \circ T') \setminus Triv$. As $T'$ is consistent, by Postulates (G1) and (G2), there exists $\mathcal{M}' \in Mod(T \circ T') \setminus Triv$. By definition of $\preceq^{T'}_T$, we have that $\mathcal{M}' \preceq^{T'}_T \mathcal{M}$, and then $\mathcal{M}' \preceq_T \mathcal{M}$ which is a contradiction. 
\item The proof of the two other conditions corresponds to the one given in Theorem~\ref{correspondence agm and orders}.  
\end{itemize}
\end{proof}
}

\Marc{Actually, the set of FA+ associated with a revision operator satisfying the AGM postulates has a lattice structure. Let $f_1,f_2$ be two FA. Let us denote $f_1 \sqcup f_2$ (resp. $f_1 \sqcap f_2$) the mapping that assigns to each knowledge base $T \subseteq Sen$ the binary relation $\preceq_T = \preceq^1_T \cup \preceq^2_T$ (resp. $\preceq_T = \preceq^1_T \cap \preceq^2_T$) where $f_i(T) = \preceq^i_T$ for $i =1,2$.}

\Marc{\begin{proposition}
\label{meet and join}
If $f_1$ and $f_2$ are FA+ for a same revision operator $\circ$, then so are $f_1 \sqcup f_2$ and $f_1 \sqcap f_2$.
\end{proposition}
\begin{proof}
It is sufficient to show that $\preceq^1_T \cup \preceq^2_T$ and $\preceq^1_T \cap \preceq^2_T$ satisfy Conditions (1) and (2) of Definition~\ref{faithful assignment} plus all the conditions of Theorem~\ref{correspondence agm and orders}. \\ 
Let us first show that they are FA. Let $T \subseteq Sen$. Let $\mathcal{M},\mathcal{M}' \in Mod(T)$. By definition of FA, then we have either $\mathcal{M} {\not \preceq^i_T} \mathcal{M}'$ and  $\mathcal{M}' {\not \preceq^i_T} \mathcal{M}$ or $\mathcal{M} \preceq^i_T \mathcal{M}'$ and  $\mathcal{M}' \preceq^i_T \mathcal{M}$ for $i = 1,2$. We then have four cases to consider, but for $f_1 \sqcap f_2(T) = \preceq_T$ (resp. $f_1 \sqcup f_2(T) = \preceq_T$), we always end up at either $\mathcal{M} {\not \preceq_T} \mathcal{M}'$ and  $\mathcal{M}' {\not \preceq_T} \mathcal{M}$ or $\mathcal{M} \preceq_T \mathcal{M}'$ and  $\mathcal{M}' \preceq_T \mathcal{M}$.  Likewise, for every $\mathcal{M} \in Mod(T)$ and every $\mathcal{M}' \in Mod \setminus Mod(T)$, we have that $\mathcal{M} \prec^i_T \mathcal{M}'$ for $i = 1,2$. Therefore, it is obvious to conclude that $\mathcal{M} \prec_T \mathcal{M}'$. \\ 
Now, by the first supplementary condition for $\preceq^1_T$ and $\preceq^2_T$ in Theorem~\ref{correspondence agm and orders}, we have for every $T' \subseteq Sen$ that $Min(Mod(T') \setminus Triv,\preceq^1_T) = Min(Mod(T') \setminus Triv,\preceq^2_T)$ \Isa{$= Mod(T \circ T') \setminus Triv$}. Hence, we can write that $Min(Mod(T') \setminus Triv,\preceq^1_T \cup \preceq^2_T) = Min(Mod(T') \setminus Triv,\preceq^1_T \cap \preceq^2_T) = Min(Mod(T') \setminus Triv,\preceq^i_T)$ for $i = 1,2$. \Isa{The three supplementary conditions are then straightforward, and} this allows us to directly conclude that  $f_1 \sqcup f_2$ and $f_1 \sqcap f_2$ are FA+.
\end{proof}
Given a revision operator $\circ$, let us denote $(\mbox{FA+($\circ$)},\leq)$ the poset of FA+ \Isa{associated with} $\circ$ where $\leq$ is the \Isa{partial} order defined by:
$$f \Marcb{\leq} g \Longleftrightarrow \forall T \subseteq Sen, f(T) \subseteq g(T)$$ \Isa{(the fact that this relation actually defines a partial order is straightforward).}
\Marcb{It is easy to show that given $f,g \in \mbox{FA+($\circ$)}$, $f \sqcup g$ (respectively $f \sqcap g$) is the least upper bound (respectively greatest \Isa{lower} bound) of $\{f,g\}$. Hence, $(\mbox{FA+($\circ$)},\leq)$ is a lattice. This lattice is further complete. Indeed, given a subset $S \subseteq \mbox{FA+($\circ$)}$, its least upper bound is the mapping $\sqcup S : T \mapsto \bigcup_{f \in S}f(T)$, and its greatest lower bound is the mapping $\sqcap S:T \mapsto \bigcap_{f \in S}f(T)$. By extending the proof of Proposition~\ref{meet and join}, it is easy to show that $\sqcup S$ and $\sqcap S$ are FA+.}
}

\subsection{Relaxation and AGM postulates}
\label{Relaxation and AGM postulates}

Relaxations have been introduced in~\cite{DAB14a,DAB14b} in the framework of description logics with the aim of defining dissimilarity between concepts. Here, we propose to generalize this notion in the framework of satisfaction systems. 

\medskip
\begin{definition}[Relaxation]
\label{relaxation}
A {\bf relaxation} is a mapping $\rho : Sen \to Sen$ satisfying: 
\begin{description}
\item[Extensivity] $\forall \varphi \in Sen, Mod(\varphi) \subseteq Mod(\rho(\varphi))$.
\item[Exhaustivity] $\exists k \in \mathbb{N}, Mod(\rho^k(\varphi)) = Mod$,
where $\rho^0$ is the identity mapping, and for all $k >0, \rho^k(\varphi) = \rho(\rho^{k-1}(\varphi))$.
\end{description} 
\end{definition}

\Marcb{Let us observe that relaxations exist if and only if the underlying satisfaction system $(Sen,Mod,\models)$ has tautologies (i.e. formulas $\varphi \in Sen$ such that $Mod(\varphi) = Mod$). Indeed, when the satisfaction system has tautologies, we can define the trivial relaxation $\rho : \varphi \mapsto \psi$ where $\psi$ is any tautology. Conversely, all relaxations imply that the underlying satisfaction system has tautologies to satisfy the exhaustivity condition.}

The interest of relaxations is that they give rise to revision operators which have demonstrated their usefulness in practice (see Sections~\ref{applications} \Isa{and~\ref{sec:theoryrelaxation}}). 

\medskip
\begin{notation}
\label{not:theoryRel}
Let $T \subseteq Sen$ be a knowledge base. Let $\mathcal{K} = \{k_\varphi \in \mathbb{N} \mid \varphi \in T\}$, and $\mathcal{K}' = \{k'_\varphi \in \mathbb{N} \mid \varphi \in T\}$. Let us note:

\begin{itemize}
\item $\rho^\mathcal{K}(T) = \{\rho^{k_\varphi}(\varphi) \mid k_\varphi \in \mathcal{K}, \varphi \in T\}$, 
\item $\sum \mathcal{K} = \sum_{k_\varphi \in \mathcal{K}} k_\varphi$,
\item  $\mathcal{K} \leq \mathcal{K}'$  when for every $\varphi \in T$, $k_\varphi \leq k'_\varphi$,
\item \Isa{$\mathcal{K} < \mathcal{K}'$ if $\mathcal{K} \leq \mathcal{K}'$ and $\exists \varphi \in T$, $k_\varphi < k'_\varphi$.}
\end{itemize}
\end{notation}
\Isa{In this notation, $k_\varphi$ is a number associated with each formula $\varphi$ of the knowledge base, which represents intuitively by which amount $\varphi$ is relaxed.}


\medskip
\begin{definition}[Revision order]
\label{revision order}
Let us define $\sqsubseteq$ the binary relation over $\mathcal{P}(Sen)$ as follows:
$$T' \sqsubseteq T''~\mbox{if}~  
\exists T''' \subseteq Sen, Mod(T''') = Mod(T'') \mbox{ and } T' \subseteq T'''.$$
\end{definition}

Intuitively, this means that $T'$ is included in $T''$ up to an equivalent knowledge base. 
The binary relation $\sqsubseteq$ will allow us to define a \Isa{coherence} criterion in the definition of revision operators (see Condition 3 in Definition~\ref{revision based in relaxation} just below). 

\medskip
\begin{definition}[Revision based on relaxation]
\label{revision based in relaxation}
Let $\rho$ be a relaxation. A  {\bf revision operator over $\rho$} is a mapping $\circ : \mathcal{P}(Sen) \times \mathcal{P}(Sen) \to \mathcal{P}(Sen)$ satisfying for every $T,T' \subseteq Sen$: 
$$T \circ T' = 
\left\{
\begin{array}{ll}
\rho^\mathcal{K}(T)  \cup T' & \mbox{if $T'$ is consistent} \\
T' & \mbox{otherwise}
\end{array}
\right.
$$
for some $\mathcal{K} = \{k_\varphi \in \mathbb{N} \mid \varphi \in T\}$ such that:
\begin{enumerate}
\item if $T'$ is consistent, then $T \circ T'$ is consistent;
\item for every $\mathcal{K}'$ such that $\rho^{\mathcal{K}'}(T)  \cup T'$ is consistent, $\sum \mathcal{K} \leq \sum \mathcal{K}'$ (minimality on the number of applications of the relaxation);
\item for every $T'' \sqsubseteq T'$, if $T \circ T'' = \rho^{\mathcal{K}'}(T) \cup T''$, then $\mathcal{K}' \leq \mathcal{K}$.
\end{enumerate}
\end{definition} 


It is important to note that given a relaxation $\rho$, several revision operators can be defined. \Marc{Without Condition 3 of Definition~\ref{revision based in relaxation}, we could accept revision operators $\circ$ \Isa{that do not satisfy}
Postulates (G5) and (G6). Hence, Condition 3 allows us to exclude such operators.} To illustrate this, let us consider in {\bf FOL} the satisfaction system $\mathcal{R} = (Sen,Mod,\models)$ over the signature $(S,F,P)$ where $S = \{s\}$, $F = \emptyset$ and $P = \{= : s \times s\}$. Let us consider $T,T' \subseteq Sen$ such that:
$$T = \left\{
\begin{array}{l}
\exists x. \exists y. (\neg x = y) \wedge \forall z (z = x \vee z = y) \\
\exists x. \exists y .\exists z. (\neg x = y \wedge \neg y = z \wedge \neg x = z) \wedge \\
\mbox{} \hfill \forall w (w = x \vee w = y \vee w = z)
\end{array}
\right\}
$$ 
$$T' = \left\{
\begin{array}{l}
\forall x. x = x \\
\forall x .\forall y. x = y \Rightarrow y = x \\
\forall x .\forall y. \forall z. x = y \wedge y = z \Rightarrow x = z
\end{array}
\right\}
$$

Obviously, $T'$ is consistent. As $T$ does not contain the axioms for equality, it is also consistent. Indeed, the model $\mathcal{M}$ with the carrier $M_s = \{0,1,2\}$ and the binary relation $=^\mathcal{M} \subseteq M_s \times M_s$ defined by $=^\mathcal{M} = \{(0,0),(1,1),(2,0)\}$ satisfies $T$. \\ 
But $T \cup T'$ is not consistent. The reason is that when the meaning of $=$ is the equality, the first axiom of $T$ can only be satisfied by models with two values while the second axiom is satisfied by models with three values. A way to retrieve the consistency is to remove one of the two axioms. This can be modeled by the relaxation $\rho$ that maps each formula to a tautology~\footnote{We will see in Section~\ref{applications} a less trivial but more interesting relaxation in {\bf FOL} that consists in changing universal quantifiers into existential ones.}. But in this case, we have then two options depending on whether we remove and change the first or the second axiom by a tautology, which give rise to two revision operators $\circ_1$ and $\circ_2$. In any cases, the first two conditions of Definition~\ref{revision based in relaxation} are satisfied by both $\circ_1$ and $\circ_2$. \\ 
Now, let us take $T'' = \{\exists x.\exists y. \neg x = y\}$ which is satisfied, when added to the axioms in $T'$, by any model with at least two elements. Hence, $(T \circ_1 T') \cup T''$ and $(T \circ_2 T') \cup T''$ are consistent. Without the third condition, nothing would prevent to define $T \circ_1 (T' \cup T'')$ (respectively $T \circ_2 (T' \cup T'')$) by removing and change in $T$ the second (respectively the first) axiom by a tautology which would be a counter-example to Postulates (G5) and (G6).  Actually, as shown by the result below, this third condition of Definition~\ref{revision based in relaxation} \Marc{entails Postulates (G5) and (G6), and then, by Proposition~\ref{proposition:reduced set of postulates}, entails Postulate (G4).} \\ 
However in some situations Condition 3 may be considered as too strong, forcing to relax more than what would be needed to satisfy only Condition 2.  This could be typically the case when Condition 2 could be obtained in two different ways, for instance for $\mathcal{K}' = \{0,1, 0, 0...\}$ or for $\mathcal{K}''=\{1,0, 0, 0...\}$. Then taking $\Marcb{Cn(}T'\Marcb{)}=\Marcb{Cn(}T''\Marcb{)}$, and revising $T\circ T'$ using $\mathcal{K}'$ and $T \circ T''$ using $\mathcal{K}''$ would not meet Condition 3. To satisfy it, relaxation should be done for instance with $\mathcal{K} = \{1,1, 0, 0...\}$. Therefore in concrete applications, we will have to find a compromise between Condition 3 and (G4)-(G6) at the price of potential larger relaxations on the one hand, and less relaxation but potentially the loss of (G4)-(G6) on the other hand.

\medskip
\begin{notation}
\Marc{In the context of Definition~\ref{revision based in relaxation},} let $T,T' \subseteq Sen$ be two knowledge bases. 
If $T \circ T' = \rho^\mathcal{K}(T) \cup T'$ with $\mathcal{K} = \{k_\varphi \in \mathbb{N} \mid \varphi \in T\}$, then let us note \Marc{$\mathcal{K}^{T'}_T = \mathcal{K}$}. 
\end{notation}

\medskip
\begin{theorem}
\label{AGM satisfied}
\Marc{Any} revision operator $\circ$ \Isa{based on a relaxation} (Definition~\ref{revision based in relaxation}) satisfies the AGM Postulates.
\end{theorem}

\begin{proof}
$\circ$ obviously satisfies Postulates (G1), (G2) and (G3). \Marc{To prove (G5)-(G6),} let us suppose $T,T',T'' \subseteq Sen$ such that $(T \circ T') \cup T''$ is consistent (the case where $(T \circ T') \cup T''$ is inconsistent is obvious). This means that $\rho^{\mathcal{K}^{T'}_T}(T) \cup T' \cup T''$ is consistent. Now, obviously we have that $T' \sqsubseteq T' \cup T''$. Hence, by the second and the third conditions of Definition~\ref{revision based in relaxation}, we necessarily have that $T \circ (T' \cup T'') = \rho^{\mathcal{K}^{T'}_T}(T) \cup T' \cup T''$, and then $Mod((T \circ T') \cup T'') = Mod(T \circ (T' \cup T''))$.
\end{proof}


\Marc{In the previous section, we showed that \Isa{several FA+ can be associated with a given} revision operator $\circ$ satisfying the AGM postulates.
Here, we define \Isa{a particular} one, which is more specific to revision operators based on relaxation. Let $\rho$ be a relaxation. Let $f_\rho$ be the mapping that associates to every $T \subseteq Sen$ the binary relation $\preceq_T$ defined as follows:}

Given $T' \subseteq Sen$, let us start by defining $\preceq^{T'}_T \subseteq Mod(T') \times Mod(T')$ as :\\
$\mathcal{M} \preceq^{T'}_T \mathcal{M}' \Longleftrightarrow$ \\
\mbox{} \hfill $\forall \mathcal{K}'' \Marc{\geq \mathcal{K}^{T'}_T}, \mathcal{M}' \in Mod(\rho^{\mathcal{K}''}(T)) \Rightarrow \exists \mathcal{K}' \Marc{\geq \mathcal{K}^{T'}_T}, 
\left\{ 
\begin{array}{l}
\mathcal{K}' < \mathcal{K}'' \mbox{ and } \\
 \mathcal{M} \in Mod(\rho^{\mathcal{K}'}(T))
 \end{array} 
 \right.$

Let us then set $\preceq_T = \bigcup_{T'} \preceq^{T'}_T$ (i.e. $\mathcal{M} \preceq_T \mathcal{M}' \Leftrightarrow \exists T', \mathcal{M} \preceq^{T'}_T \mathcal{M}'$). Let us note that $\preceq_T \subseteq Mod \times Mod$ because $\preceq^\emptyset_T \subseteq \preceq_T$.

\Isa{Intuitively, it means that $T$ has to be relaxed more to be satisfied by $\mathcal{M}'$ than to be satisfied by $\mathcal{M}$.}

\medskip
\begin{theorem}
\label{is a FA}
\Marc{For any revision operator $\circ$ based on a relaxation $\rho$ as defined in Definition~\ref{revision based in relaxation}, the mapping $f_\rho$ is a FA+.}
\end{theorem}

\begin{proof}
\Marc{Let $T \subseteq Sen$.} Let us first show that $\Marc{f_\rho(T) =} \preceq_T$ is faithful. 
\begin{itemize}
\item Obviously, we have for every $\mathcal{M},\mathcal{M}' \in Mod(T)$ and every $T' \subseteq Sen$ that both $\mathcal{M} {\not \preceq^{T'}_T} \mathcal{M}'$ and $\mathcal{M}' {\not \preceq^{T'}_T} \mathcal{M}$. Hence the same relations hold for $\preceq_T$.
\item Let $\mathcal{M} \in Mod(T)$ and let $\mathcal{M}' \in Mod \setminus Mod(T)$. Obviously, we have that $\mathcal{M} \preceq^\emptyset_T \mathcal{M}'$. Let $T' \subseteq Sen$ such that $\mathcal{M},\mathcal{M}' \in Mod(T')$ (the case where \Marc{for all $T' \subseteq Sen$} $\mathcal{M}$ or $\mathcal{M}'$ is not in $Mod(T')$ implies that $\mathcal{M}$ and $\mathcal{M}'$ are incomparable by $\preceq^{T'}_T$, and then we directly have that $\mathcal{M}' {\not \preceq_T} \mathcal{M}$). Here two cases have to be considered:
\begin{enumerate}
\item $\mathcal{M} \in Triv$. As $\mathcal{M}' {\not \in} Mod(T)$, then $\mathcal{M}' {\not \in} Triv$. Hence, there does not exist $\mathcal{K}' < \mathcal{K}$ such that $\mathcal{M}' \in Mod(\rho^{\mathcal{K}'}(T))$. Otherwise, $\rho^{\mathcal{K}'}(T) \cup T'$ would be consistent, which would contradict the hypothesis that $T \circ T' = \rho^\mathcal{K}(T) \cup T'$.
\item $\mathcal{M} {\not \in} Triv$.  We have that $\mathcal{M} \in Mod(T \cup  T')$ but $\mathcal{M}' {\not \in} Mod(T \cup T')$, and then $\mathcal{M}' {\not \preceq^{T'}_T} \mathcal{M}$ \Marc{By definition of $\circ$}. 
\end{enumerate}
Hence, in both cases we can conclude that $\mathcal{M}' {\not \preceq_T} \mathcal{M}$.
\end{itemize}
Let us prove that $Mod(T \circ T') \setminus Triv = Min(Mod(T') \setminus Triv,\preceq_T)$. This will directly prove that $Min(Mod(T') \setminus Triv,\preceq_T) \neq \emptyset$ when $T'$ is consistent. Indeed, by definition, we have that $T \circ T'$ is consistent when $T'$ is consistent, and then $Min(Mod(T') \setminus Triv,\preceq_T)  \neq \emptyset$ if $Mod(T \circ T') \setminus Triv = Min(Mod(T') \setminus Triv,\preceq_T)$. \\
If $T'$ is inconsistent, then so is $T \circ T'$ by definition. Hence, $Mod(T \circ T') \setminus Triv = Min(Mod(T') \setminus Triv,\preceq_T) = \emptyset$.\\ 
Let us now suppose that $T'$ is consistent.
\begin{itemize}
\item Let us show that $Mod(T \circ T') \setminus Triv  \subseteq Min(Mod(T') \setminus Triv,\preceq_T)$. Let $\mathcal{M} \in Mod(T \circ T') \setminus Triv$. Let $\mathcal{M}' \in Mod(T') \setminus Triv$. Two cases have to be considered:
\begin{enumerate}
\item $\mathcal{M}' \in Mod(T \circ T')$. Obviously, we have both $\mathcal{M} {\not \preceq^{T'}_T} \mathcal{M}'$ and $\mathcal{M}' {\not \preceq^{T'}_T} \mathcal{M}$. Let us show that this is also true for every $T'' \subseteq Sen$ such that $\mathcal{M},\mathcal{M}' \in Mod(T'')$. Let us suppose that there exists $T'' \subseteq Sen$ such that $\mathcal{M}' \preceq^{T''}_T \mathcal{M}$. By hypothesis, we then have that $(T \circ T') \cup T''$ is consistent. Therefore, by Conditions 2 and 3 of Definition~\ref{revision based in relaxation}, we have that $(T \circ T') \cup T'' = T \circ (T' \cup T'')$. Hence, we also have that $T \circ (T' \cup T'') =  \Marc{\rho^{\mathcal{K}^{T'}_T}}(T) \cup T' \cup T''$. Consequently, \Marc{as $T'' \sqsubseteq T' \cup T''$}, we have by Condition 3 of Definition~\ref{revision based in relaxation} that \Marc{$\mathcal{K}^{T''}_T \leq \mathcal{K}^{T'}_T$}. Therefore, as $\mathcal{M}' \preceq^{T''}_T \mathcal{M}$, we can deduce that there exists $\mathcal{K}''< \mathcal{K}^{T'}_T$ such that $\mathcal{M}' \in Mod(\rho^{\mathcal{K}''}(T))$. We then have that $\rho^{\mathcal{K}''}(T) \cup T'$ is consistent, and then by Condition 2 of Definition~\ref{revision based in relaxation}, $\sum \Marc{\mathcal{K}^{T'}_T} \leq \sum \mathcal{K}''$, which is a contradiction.
\item $\mathcal{M}' {\not \in} Mod(T \circ T')$. By definition of $\preceq^{T'}_T$, we have that $\mathcal{M} \preceq^{T'}_T \mathcal{M}'$, and therefore $\mathcal{M} \preceq_T \mathcal{M}'$.
\end{enumerate}
Finally, we can conclude that $\mathcal{M} \in Min(Mod(T') \setminus Triv,\preceq_T)$.
\item Let us now show that $Min(Mod(T') \setminus Triv,\preceq_T) \subseteq Mod(T \circ T') \setminus Triv$. Let $\mathcal{M} \in Min(Mod(T') \setminus Triv,\preceq_T)$. Let us suppose that $\mathcal{M} {\not \in} Mod(T \circ T') \setminus Triv$. As $T'$ is consistent, then so is $T \circ T'$. Hence, there exists $\mathcal{M}' \in Mod(T \circ T') \setminus Triv$. As $\mathcal{M} \in Mod(T') \setminus Mod(T \circ T')$, we have that $\mathcal{M}' \preceq^{T'}_T \mathcal{M}$, and then as $\mathcal{M} \in Min(Mod(T') \setminus Triv,\preceq_T)$ we also have that $\mathcal{M} \preceq_T \mathcal{M}'$. This means that there exists $T'' \subseteq Sen$ such that $\mathcal{M},\mathcal{M}' \in Mod(T'')$ and $\mathcal{M} \preceq^{T''}_T \mathcal{M}'$. By hypothesis, we then have that $(T \circ T') \cup T''$ is consistent. Therefore, by Conditions 2 and 3 of Definition~\ref{revision based in relaxation}, we have that $(T \circ T') \cup T'' = T \circ (T' \cup T'')$. Hence, we also have that $T \circ (T' \cup T'') =  \Marc{\rho^{\mathcal{K}^{T'}_T}}(T) \cup T' \cup T''$. Consequently, we have by Condition 3 of Definition~\ref{revision based in relaxation} that \Marc{$\mathcal{K}^{T''}_T \leq \mathcal{K}^{T'}_T$}. Hence, there exists \Marc{$\mathcal{K}'' \geq \mathcal{K}^{T''}_T$} such that $\mathcal{K}'' < \Marc{\mathcal{K}^{T'}_T}$ and $\mathcal{M} \in Mod(\rho^{\mathcal{K}''}(T))$. We can then deduce that $\rho^{\mathcal{K}''}(T) \cup T'$ is consistent, and then by Condition 2 of Definition~\ref{revision based in relaxation} we have that $\sum \Marc{\mathcal{K}^{T'}_T} \leq \sum \mathcal{K}''$, which is a contradiction. 
\end{itemize}
Finally, to prove the last point, we follow the same steps as in the proof of Theorem~\ref{correspondence agm and orders}. 
\end{proof}



\subsection{Applications}
\label{applications}

In this section, we illustrate our general approach by defining revision operators based on relaxations for the logics {\bf PL}, {\bf HCL}, and {\bf FOL}. 
\Isa{We further develop the case of DLs in Section~\ref{sec:theoryrelaxation}}, by defining several concrete relaxation operators for different fragments of the DL $\mathcal{ALC}$.

\subsubsection{Revision in {\bf PL}}

Here, drawing inspiration from Bloch \& al.'s works in~\cite{BL02,BPU04} on Morpho-Logics, we define relaxations based on dilations from mathematical morphology~\cite{BHR06}. It is well established in {\bf PL} that knowing a formula is equivalent to knowing the set of its models. Hence, we can identify any propositional formula $\varphi$ with the set of its interpretations $Mod(\varphi)$. To define relaxations in {\bf PL}, we will apply set-theoretic morphological operations. First, let us recall basic definitions of dilation in mathematical morphology~\cite{BHR06}. Let $X$ and $B$ be two subsets of $\mathbb{R}^n$. The dilation of $X$ by the structuring element $B$, denoted by $D_B(X)$, is defined as follows:
$$D_B(X) = \{x \in \mathbb{R}^n \mid B_x \cap X \neq \emptyset\}$$
where $B_x$ denotes the translation of $B$ at $x$. 
More generally, dilations in any space can be defined in a similar way by considering the structuring element as a binary relationship between elements of this space.

In {\bf PL}, this leads to the following dilation of a formula $\varphi \in Sen$:
$$Mod(D_B(\varphi)) = \{\nu \in  Mod(\Sigma) \mid B_\nu \cap Mod(\varphi) \neq \emptyset\}$$
where $B_\nu$ contains all the models that satisfy some relationship with $\nu$. The relationship standardly used is based on a discrete distance $\delta$ between models, and the most commonly used is the Hamming distance $d_H$ where $d_H(\nu,\nu')$ for two propositional models over a same signature is the number of propositional symbols that are instantiated differently in $\nu$ and $\nu'$. From any distance $\delta$ between models, a distance from models to a formula is derived as follows: $d(\nu,\varphi) = min_{\nu' \models \varphi}\delta(\nu,\nu')$. In this case, we can rewrite the dilation of a formula as follows:
$$Mod(D_B(\varphi)) = \{\nu \in  Mod(\Sigma) \mid d(\nu,\varphi) \leq 1\}$$
This consists in using the distance ball of radius 1 as structuring element. To ensure the exhaustivity condition to our relaxation, we need to add a condition on distances, the {\em betweenness property}~\cite{DAB14a}. 

\medskip
\begin{definition}[Betweenness property]
Let $\delta$ be a discrete distance over a set $S$. $\delta$ has the {\bf betweenness property} if for all $x,y \in S$ and all $k \in \{0,1,\ldots,\delta(x,y)\}$, there exists $z \in S$ such that 
$\delta(x,z) = k$ and $\delta(z,y) = \delta(x,y) - k$. 
\end{definition}
The Hamming distance trivially satisfies the betweenness property. The interest for our purpose of this property is that it allows from any model to reach any other one, and then ensuring the exhaustivity property of relaxation~\footnote{Hence, dilation of formulas could also be defined by using a distance ball of radius $n$ as structuring element~\cite{BL02}.}.  

\medskip
\begin{proposition}
\label{relaxation for PL}
The dilation $D_B$ is a relaxation when it is applied to formulas $\varphi \in Sen$ 
for a finite signature,
and it is based on a distance between models that satisfies the betweenness property. 
\end{proposition}

\begin{proof}
It is extensive. Indeed, for every $\varphi$ and for every model $\nu \in Mod(\varphi)$, we have that $d(\nu,\varphi) = 0$, and then $\varphi \models D_B(\varphi)$. Exhaustivity results from the fact that the considered signature is a finite set and from the betweenness property. 
\end{proof}

\subsubsection{Revision in {\bf HCL}}

Many works have focused on belief revision involving propositional Horn formulas (cf. \cite{DP15} to have an overview on these works). Here, we propose to extend relaxations that we  have defined in the framework of {\bf PL} to deal with the Horn fragment of propositional theories. First, let us introduce some notions.

\medskip
\begin{definition}[Model intersection]
Given a propositional signature $\Sigma$ and two $\Sigma$-models $\nu,\nu': \Sigma \to \{0,1\}$, we note $\nu \cap \nu' : \Sigma \to \{0,1\}$ the $\Sigma$-model defined by: 
$$p \mapsto 
\left\{ 
\begin{array}{ll}
1 & \mbox{if $\nu(p) = \nu'(p) = 1$} \\
0 & \mbox{otherwise}
\end{array}
\right.$$
Given a set of $\Sigma$-models $\mathcal{S}$, we note 
$$cl_\cap(\mathcal{S}) = \mathcal{S} \cup \{\nu \cap \nu' \mid \nu,\nu' \in \mathcal{S}\}$$
\end{definition}
$cl_\cap(\mathcal{S})$ is then the closure of $\mathcal{S}$ under intersection of positive atoms. 

It is well-known that for any set $\mathcal{S}$ closed under intersection of positive atoms, there exists a Horn sentence $\varphi$ that defines $\mathcal{S}$ (i.e. $Mod(\varphi) = \mathcal{S}$). Given a distance $\delta$ between models, we then define a relaxation $\rho$ as follows: for every Horn formula $\varphi$, $\rho(\varphi)$ is any Horn formula $\varphi'$ such that $Mod(\varphi') = cl_\cap(Mod(D_B(\varphi))$ (by the previous property, we know that such a formula $\varphi'$ exists). 

\medskip
\begin{proposition}
With the same conditions as in Proposition~\ref{relaxation for PL}, the mapping $\rho$ is a relaxation. 
\end{proposition} 

\subsubsection{Revision in {\bf FOL}}
\label{revision in FOL}

A trivial way to define a relaxation in {\bf FOL} is to map any formula to a tautology. A less trivial and more interesting relaxation is to change universal quantifiers to existential ones. Indeed, given a formula $\varphi$ of the form $\forall x. \psi$. If $\varphi$ is not consistent with a given theory $T$, $\exists x. \psi$ may be consistent with $T$ (if it cannot be consistent for all values, it can be for some of them). In the following we suppose that given a signature, every formula $\varphi \in Sen$ is a disjunction of formulas in prenex form (i.e. $\varphi$ is of the form $\bigvee_{j}Q^j_1x^j_1 \ldots Q^j_{n_j} x^j_{n_j}. \psi_j$ where each $Q^j_i$ is in $\{\forall,\exists\}$). Let us define the relaxation $\rho$ as follows, for a tautology $\tau$:
\begin{itemize}
\item $\rho(\tau) = \tau$;
\item $\rho(\exists_1 x_1 \ldots \exists_n x_n. \varphi) = \tau$;
\item Let $\varphi = Q_1x_1 \ldots Q_n x_n. \psi$ be a formula such that the set $E_\varphi = \{i, 1 \leq i \leq n \mid Q_i = \forall\} \neq \emptyset$. Then, $\rho(Q_1x_1 \ldots Q_n x_n. \varphi) = \bigvee_{i \in E_\varphi} \varphi_i$ where $\varphi_i = Q'_1x_1 \ldots Q'_n x_n.\psi$ such that for every $j \neq i$, $1 \leq j \leq n$, $Q'_j = Q_j$ and $Q'_i = \exists$;
\item $\rho(\bigvee_{j}Q^j_1x^j_1 \ldots Q^j_{n_j} x^j_{n_j}. \psi) = \bigvee_{j}\rho(Q^j_1x^j_1 \ldots Q^j_{n_j} x^j_{n_j}. \psi)$.
\end{itemize}

\medskip
\begin{proposition}
$\rho$ is a relaxation. 
\end{proposition}

\begin{proof}
It is obviously extensive, and exhaustivity results from the fact that in a finite number of steps, we always reach the tautology $\tau$. 
\end{proof}

\section{Relaxation of theories and  associated revision operator in DL}
\label{sec:theoryrelaxation}
%
%

Our idea to define revision operators is to relax the set of models of the old belief until it becomes consistent with the new pieces of knowledge. This is illustrated in Figure~\ref{fig:relaxation} where theories are represented as sets of their models. Intermediate steps to define the revision operators are then the definition of  formula  and theory relaxations.  The whole scheme of our framework is provided in Figure~\ref{fig:structure}. 

\begin{figure}[htbp]
  \newlength{\ddx}
  \newlength{\ddy}
  \begin{center}
    \begin{tikzpicture}
  \setlength{\ddx}{0.6cm}
  \setlength{\ddy}{.125cm}
  \newlength{\shift}
  \setlength{\shift}{.5cm}
  \newlength{\cshift}
  \setlength{\cshift}{.15cm}
  \path[draw, fill, fill opacity = 0.3] (\ddx, -6.5\ddy) ..
    controls (1.5\ddx, -6.5\ddy) and (1.5\ddx, -9\ddy)
    .. (1.5\ddx, -10\ddy) ..
    controls (1.5\ddx, -12\ddy) and (1.5\ddy, -12\ddy)
    .. (0, -12\ddy) ..
    controls (-\ddx, -12\ddy) and (-\ddx, -12\ddy)
    .. (-\ddx, -10\ddy) ..
    controls (-\ddx, -9\ddy) and (-\ddx, -6.5\ddy)
    .. (\ddx, -6.5\ddy);
  \node at (-1.5\ddx, -10\ddy) {$T'$};
  \path[draw, fill, fill opacity = 0.1] (0, 0) ..
    controls (.5\ddx, 0) and (.5\ddx, 2\ddy)
    .. (\ddx, 2\ddy) ..
    controls (1.5\ddx, 2\ddy) and (2\ddx, \ddy)
    .. (2\ddx, 0) ..
    controls (2\ddx, -\ddy) and (\ddx, -3\ddy)
    .. (0, -3\ddy) ..
    controls (-\ddx, -3\ddy) and (-2\ddx, -2.5\ddy)
    .. (-2\ddx, -\ddy) ..
    controls (-2\ddx, 0) and (-1.5\ddx, \ddy)
    .. (-\ddx, \ddy) ..
    controls (-.5\ddx, \ddy) and (-.5\ddx, 0)
    .. (0, 0);
  \node at (-1.1\ddx, 1.7\ddy) {$T$};
  \path[draw, dashed] (0, \shift) ..
    controls (.5\ddx - \cshift, \shift) and (.5\ddx - \cshift, 2\ddy + \shift)
    .. (\ddx, 2\ddy + \shift) ..
    controls (1.5\ddx + \cshift, 2\ddy + \shift) and
      (2\ddx + \shift, \ddy + 2\cshift)
    .. (2\ddx + \shift, 0) ..
    controls (2\ddx + \shift, -\ddy - 2\cshift) and
      (\ddx + 2\cshift, -3\ddy - \shift)
    .. (0, -3\ddy - \shift) ..
    controls (-\ddx - 2\cshift, -3\ddy - \shift) and
      (-2\ddx - \shift, -2.5\ddy - 2\cshift)
    .. (-2\ddx - \shift, -\ddy) ..
    controls (-2\ddx - \shift, 2\cshift) and
      (-1.5\ddx - \cshift, \ddy + \shift)
    .. (-\ddx, \ddy + \shift) ..
    controls (-.5\ddx - \cshift, \ddy + \shift) and (-.5\ddx, \shift)
    .. (0, \shift);
  \node at (-1.1\ddx, 1.7\ddy + \shift) {$\rho^{\K^1}(T)$};
  \setlength{\cshift}{.3cm}
  \setlength{\shift}{1cm}
  \path[draw, dashed] (0, \shift) ..
    controls (.5\ddx - \cshift, \shift) and (.5\ddx - \cshift, 2\ddy + \shift)
    .. (\ddx, 2\ddy + \shift) ..
    controls (1.5\ddx + \cshift, 2\ddy + \shift) and
      (2\ddx + \shift, \ddy + 2\cshift)
    .. (2\ddx + \shift, 0) ..
    controls (2\ddx + \shift, -\ddy - 2\cshift) and
      (\ddx + 2\cshift, -3\ddy - \shift)
    .. (0, -3\ddy - \shift) ..
    controls (-\ddx - 2\cshift, -3\ddy - \shift) and
      (-2\ddx - \shift, -2.5\ddy - 2\cshift)
    .. (-2\ddx - \shift, -\ddy) ..
    controls (-2\ddx - \shift, 2\cshift) and
      (-1.5\ddx - \cshift, \ddy + \shift)
    .. (-\ddx, \ddy + \shift) ..
    controls (-.5\ddx - \cshift, \ddy + \shift) and (-.5\ddx, \shift)
    .. (0, \shift);
  \node at (-2.1\ddx, 1\ddy + \shift) {$\rho^{\K^2}(T)$};
\end{tikzpicture}
  \end{center}
  \caption{Relaxations of $T$ until it becomes consistent with $T'$. \label{fig:relaxation}}
\end{figure}
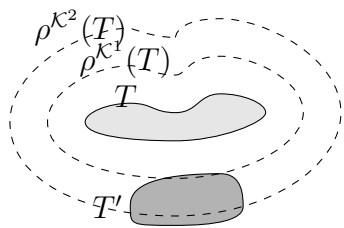

{\footnotesize
 \begin{figure}[htbp]
  \begin{center}
    \tikzstyle{block} = [rectangle, draw, fill=black!10,
                         rounded corners, minimum height=3em]
    \resizebox{\linewidth}{!}{
      \begin{tikzpicture}[node distance = 5cm, auto, , align=center]
        \node[block, text width=3cm] (metric)
          {Concept relaxation \\ $\rho \colon \mathsf{C} \to \mathsf{C}$\\
          Concept retraction \\
          $\kappa \colon \mathsf{C}  \to \mathsf{C}$};
        \node[block, text width=3.9cm, right of = metric] (dilation)
          {Formula relaxation \\ $\rho_F \colon  Sen \to Sen$};
        \node[block, text width=3cm, right of = dilation] (relaxation)
          {Theory relaxation \\ 
          $\rho^{\cal K}$};
       \node[block, text width=3.6cm, right of = relaxation] (dissimilarity)
          {Revision \\ $\circ \colon \mathcal{P}(Sen) \times \mathcal{P}(Sen)  \to \mathcal{P}(Sen)$};
        \path[->]
          (metric) edge node[below] { Def.\\\ref{def:formula_dilation},\ref{def:formula_erosion} } (dilation)
          (dilation) edge node[below] {Not.\\ \ref{not:theoryRel}} (relaxation)
          (relaxation) edge node[below] {Def.\\ \ref{revision based in relaxation}\\ Th.\\ \ref{AGM satisfied}}
            (dissimilarity);
      \end{tikzpicture}
    }
  \end{center}
  \caption{From concept relaxation and retraction to  revision operators in DL.\label{fig:structure}}
\end{figure}
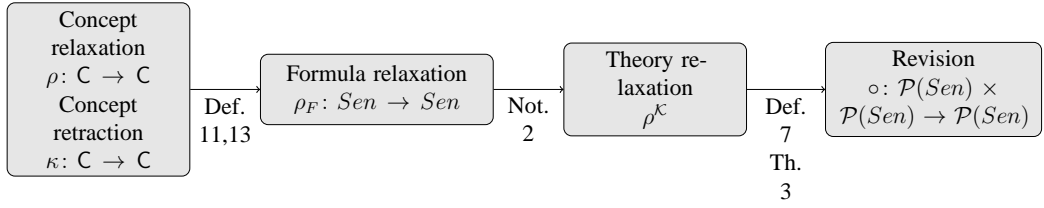
}

\subsection{Concept relaxation}

As already explained in Section~\ref{Relaxation and AGM postulates}, relaxation has been introduced in~\cite{DAB14a,DAB14b}. It has been first defined over concepts, and then extended to formulas. In~\cite{DAB14a,DAB14b}, concept relaxation is defined as follows:

\medskip
\begin{definition}[Concept relaxation]
\label{concept relaxation}
Given a signature $(N_C,N_R,I)$, we note $\mathsf{C}$ the set of concepts over this signature. A {\bf concept relaxation} $\rho : \mathsf{C} \to \mathsf{C}$ is a mapping that satisfies, \Isa{for all $C$ in $\mathsf{C}$}:
\begin{enumerate}
\item $C \sqsubseteq \rho(C)$
\item $\exists k \in \mathbb{N}, \top \sqsubseteq \rho^k(C)$
\end{enumerate}
\end{definition}
\Isa{Note that the non-decreasingness property in the original definition of a concept relaxation in~\cite{DAB14b} is removed here, since it is not needed in our construction.}

A trivial concept relaxation is the operation $\rho_\top$ that maps every concept $C$ to $\top$. Other non-trivial concrete concept relaxations such as the one that changes universal quantifiers to existential ones as in {\bf FOL} \Isa{will be detailed next.}


\subsection{Formula relaxation and theory relaxation}

A formula relaxation $\rho_F$ in DL is defined as in Definition~\ref{relaxation}. From a formula relaxation, we can define a theory relaxation $\rho^{\mathcal{K}}$ as in Notation~\ref{not:theoryRel}.

\Isa{In the satisfaction system DL, let $\rho$ be a relaxation for DL. Then from $\rho$ we define a revision operator as in Definition~\ref{revision based in relaxation}. According to Theorem~\ref{AGM satisfied}, it satisfies the AGM postulates.}

\Isa{As mentioned above,} Condition 3 in 
\Isa{Definition~\ref{revision based in relaxation}} may be considered as too strong in many real world applications. \Isa{This may be the case in particular} in the context of ontological engineering, where one may want to change only one axiom (or a limited number of axioms) instead of the whole theory. We will come back to this point when we will introduce a first example to illustrate relaxation operators  in the DL \EL{}.

In the following, we introduce concrete relaxation operators suited to the syntax of  the logic $\ALC{}$, as defined in Section~\ref{sec:BasicDef}, and its fragments $\EL{}$ and $\ELext{}$. $\EL{}$-concept description constructors are existential restriction ($\exists$), conjunction ($\sqcap$), $\top$ and $\bot$, while $\ELext{}$-concept constructors are those of $\EL{}$ enriched with disjunction ($\sqcup$).

\subsection{Abstract relaxation and retraction operators}
We propose to define a formula relaxation in two ways (other definitions may also exist). For sentences of the form $C\sqsubseteq D$, the first proposed approach consists in relaxing the set of models of $D$ while the second one amounts to  ``retract'' the set of models of $C$.  

From any concept relaxation $\rho$, we can define 1 relaxation on formulas
We suppose that any signature $(N_C,N_R,I)$ always contains in $N_R$ a relation name $r_\top$ the meaning of which is in any model $\mathcal{O}$, $r_\top^\mathcal{O} = \Delta^\mathcal{O} \times \Delta^\mathcal{O}$.
%
%
%

\medskip
\begin{definition}[Formula relaxation based on concept relaxation]
\label{def:formula_dilation}
Let $\rho$  a concept relaxation as in Definition~\ref{concept relaxation}. \Isa{A \emph{\textbf{formula relaxation based on $\rho$}}, denoted $\presuper{r}{\rho_F}$ is defined as follows, for any two complex concepts $C$ and $D$, any individuals $a, b$, and any role $r$:}
\[
\begin{aligned}
\presuper{r}{\rho_F}(C\sqsubseteq D) &\equiv C \sqsubseteq \rho(D), \\
\presuper{r}{\rho_F}(a: C) &\equiv a : \rho(C), \\
\presuper{r}{\rho_F}(\langle a,b\rangle : r)) &\equiv \langle a,b\rangle : r_\top .
\end{aligned}
\]
\end{definition}

\medskip
\begin{proposition}
$\presuper{r}{\rho_F}$ is a formula relaxation, i.e. extensive and exhaustive.
\end{proposition}
\begin{proof}
It directly follows from the extensivity and exhaustivity of $\rho$.
\end{proof}

Another strategy for defining a formula relaxation consists in retracting the concept in the left hand side of a sentence of the form $C\sqsubseteq D$. Before providing this definition  we need to formalize this notion of retraction, which could  be seen as an anti-relaxation.

\medskip\begin{definition}[Concept retraction]
\label{concepterosion}
A \emph{\textbf{(concept) retraction}} is an operator $\kappa \colon \mathsf{C}
  \rightarrow \mathsf{C}$ that satisfies the following two properties for all
  $C$ in $\mathsf{C}$:
  \begin{enumerate}
    \item $\kappa$ is \emph{anti-extensive}, i.e.\ $\kappa(C) \sqsubseteq C $, and 
    \item $\kappa$ is \emph{exhaustive}, i.e. $\forall D\in \mathsf{C}, \exists k \in \mathbb{N} \text{ such that } \kappa^k(C) \sqsubseteq D$.
  \end{enumerate}
\end{definition} 


\medskip
\begin{definition}[Formula relaxation based on concept retraction]
\label{def:formula_erosion}
\Isa{ A \emph{\bf formula relaxation based on a concept retraction $\kappa$}, denoted $\presuper{c}{\rho_F}$, is defined as follows, for any two complex concepts $C$ and $D$, any individuals $a, b$, and any role $r$:}
\[
\begin{aligned}
\presuper{c}{\rho_F}(C\sqsubseteq D) & \equiv \kappa(C) \sqsubseteq D,\\
\presuper{c}{\rho_F}(a: C) &\equiv a:\top,\\
\presuper{c}{\rho_F}(\langle a,b\rangle : r)) &\equiv \langle a,b\rangle : r_\top.
\end{aligned}
\]   
\end{definition}

A similar construction can be found in~\cite{QLB06} for sentences of the form $(a: C)$. 

\medskip
\begin{proposition}
$\presuper{c}{\rho_F}$ is a formula relaxation.
\end{proposition}
\begin{proof}
\Isa{Extensivity and exhaustivity} follow directly from the properties of $\kappa$.
\end{proof}

For coming up with revision operators, it remains to define concrete relaxation and retraction operators at the concept level, according to our general schema in Figure~\ref{fig:structure}. Some examples of retraction and relaxation operators are given below for $\EL{}$ and $\mathcal{ELU}$, respectively.

\subsection{Relaxation and retraction in $\mathcal{EL}$}

\paragraph{$\EL{}$-Concept Retractions.}
A trivial concept retraction is the operator $\kappa_{\bot}$ that maps every concept to $\bot$. This operator is particularly interesting for debugging ontologies expressed in $\mathcal{EL}$~\cite{schlobach2007debugging}. Let us illustrate this operator through the following example adapted from~\cite{QLB06} to restrict the language to $\mathcal{EL}$.

\medskip
\begin{example}
\label{ex:tweety}
Let $T=\{\textsc{Tweety}\sqsubseteq \textsc{bird}, \textsc{bird}\sqsubseteq \textsc{flies}\}$ and $T'=\{\textsc{Tweety} \sqcap \textsc{flies} \sqsubseteq \bot\}$. Clearly $T \cup T'$ is inconsistent. The formula relaxation \Isa{based on the retraction $\kappa_{\bot}$} amounts to apply $\kappa_\bot$ to the concept $\textsc{Tweety}$ resulting in the following new knowledge base $\{\bot \sqsubseteq \textsc{bird}, \textsc{bird}\sqsubseteq \textsc{flies}\}$ which is now consistent with $T'$. An alternative solution is to retract the concept $\textsc{bird}$ in $\textsc{bird}\sqsubseteq \textsc{flies}$ which results in the following knowledge base $\{\textsc{Tweety} \sqsubseteq \textsc{bird}, \bot \sqsubseteq \textsc{flies}\}$ which is also consistent with $T'$. The  sets of minimal sum $\mathcal{K}_1$ and $\mathcal{K}_2$  in Condition 2 of \Isa{Definition~\ref{revision based in relaxation}} are $\mathcal{K}_1=\{1,0\}$, (i.e. $k_{\varphi_1}=1, k_{\varphi_2}=0$, where $\varphi_1=\textsc{Tweety}\sqsubseteq \textsc{bird}, \varphi_2= \textsc{bird}\sqsubseteq \textsc{flies}$) and  $\mathcal{K}_2=\{0,1\}$. However, to ensure Condition 3 of the same \Isa{definition}, we must relax one more time the  axioms in $T$ leading to the following knowledge base $\{\bot \sqsubseteq \textsc{bird}, \bot\sqsubseteq \textsc{flies}\}$. The final revision then writes $T\circ T'=\{\bot \sqsubseteq \textsc{bird}, \bot\sqsubseteq \textsc{flies},\textsc{Tweety} \sqcap \textsc{flies} \sqsubseteq \bot \}$.  This revision satisfies the AGM postulates (G1-G6) but may appear too strong, and one may prefer one of the following solutions:  $T\circ_1 T'=\{\bot \sqsubseteq \textsc{bird}, \textsc{bird}\sqsubseteq \textsc{flies},\textsc{Tweety} \sqcap \textsc{flies} \sqsubseteq \bot \}$ or  $T\circ_2 T'=\{\textsc{Tweety} \sqsubseteq \textsc{bird}, \bot\sqsubseteq \textsc{flies},\textsc{Tweety} \sqcap \textsc{flies} \sqsubseteq \bot \}$  at the price of loosing (G4)-(G6).
\end{example}

\Jamal{Although the results are rather intuitive, one should note that it is pretty hard to figure out what each DL researcher would like to have as a result in such an example, and this enforces the interest of relying on an established theory such as AGM or its extension. In our work we propose operators enjoying a bunch of properties stemming from our adaptation of the AGM theory. Some of them can meet the requirement of a knowledge engineer, and some other may not completely, depending on the context, the ontology, etc.}

\paragraph{$\EL{}$-Concept Relaxations.}
Dually, a trivial relaxation is the operator $\rho_\top$ that maps every concept to $\top$. Other non-trivial $\EL{}$-concept description relaxations have been introduced in~\cite{DAB14a}. We summarize here some of these operators.

\EL{} concept descriptions can appropriately be
represented as labeled trees, often called \emph{\EL{} description trees}
\cite{baader1999computing}. An \EL{} description tree is a tree whose nodes are
labeled with sets of concept names and whose edges are labeled with role names.
An \EL{} concept description
\begin{equation}
  \label{eqn:normalForm}
  C \equiv P_1 \sqcap \cdots \sqcap P_n
  \sqcap \exists r_1. C_1 \sqcap \cdots \sqcap \exists r_m. C_m
\end{equation}
with $P_i \in N_C \cup \{\top\}$, can be translated into a description tree by
labeling the root node $v_0$ with $\{P_1, \dots, P_n\}$, creating an $r_j$
successor, and then proceeding inductively by expanding $C_j$ for the
$r_j$-successor node for all $j \in \{1, \dots, m\}$. 

An $\EL{}$-concept description relaxation then amounts  to apply simple tree  operations. 
Two relaxations can hence be defined~\cite{DAB14a}: (i) $\rho_\text{depth}$ that
reduces the role depth of each concept by $1$, simply by pruning the
description tree, and (ii) $\rho_\text{leaves}$ that removes all leaves from a description
tree.

\subsection{Relaxations in $\mathcal{ELU}$}
The relaxation defined above exploits the strong property that an $\EL{}$ concept description is isomorphic to a description tree. This is arguably not true for more expressive DLs. Let us try to go  one step further in expressivity and consider the logic $\ELext$. 
A relaxation operator as introduced in~\cite{DAB14a} requires a concept description to be in a special normal form, called normal form with grouping of existentials, defined recursively as follows.

\medskip
\begin{definition}[Normal form with grouping of existential restrictions]
  We say that an $\EL$-concept $D$ is written in \emph{\bf normal
  form with grouping of existential restrictions} if it is of the form
  \begin{equation}
  \label{eqn:normalFormWithGrouping}
    D = \bigsqcap_{A \in N_D} A \sqcap \bigsqcap_{r \in N_R} D_{r},
  \end{equation}
  where $N_D \subseteq N_C$ is a set of concept names and the concepts $D_{r}$
  are of the form
  \begin{equation}
  \label{eqn:groupedExistentials}
    D_{r} = \bigsqcap_{E \in \Cmc_{D_r}} \exists r.E,
  \end{equation}
  where no subsumption relation holds between two distinct conjuncts and
  $\Cmc_{D_r}$ is a set of complex $\EL$-concepts that are themselves in
normal form with grouping of existential restrictions. 
\end{definition}

The purpose of
  $D_r$ \new{terms} is simply to group existential restrictions that share the
  same role name.
  For an \ELext-concept $C$ we say that $C$ is in \emph{normal form} if it is
  of the form ($C\equiv C_1\sqcup C_2\sqcup \cdots \sqcup C_k$) and each of the $C_i$ is an
  \EL-concept in normal form with grouping of existential restrictions.

\medskip\begin{definition}[Relaxation from normal form~\cite{DAB14a}]
Given an $\ELext$-concept description $C$ we define an operator $\rho_e$
recursively as follows.
\begin{itemize}
\item For $C = \top$ we define $\rho_e(C) =  \top$,
\item For $C = D_r$, where $D_r$ is a group of existential restrictions as in
Equation~\ref{eqn:groupedExistentials}, we need to distinguish two cases:
\begin{itemize}
  \item if $D_r \equiv \exists r.\top$ we define $\rho_e(D_r) = \top$, and
  \item if $D_r \not \equiv \exists r.\top$ then we define
$      \rho_e(D_r) = \bigsqcup_{\Smc \subseteq \Cmc_{D_r}}
        \left(
          \bigsqcap_{E \notin \Smc}
            \exists r. E \sqcap
            \exists r.\rho_e \bigg(\bigsqcap_{F \in \Smc} F\bigg)
        \right)$.
\end{itemize}
Note that in the latter case $\top \notin \Cmc_{D_r}$ since $D_r$ is
in normal form. 
\item For $C = D$ as in Equation~\ref{eqn:normalFormWithGrouping} we define
$  \rho_e(D) = \bigsqcup_{G \in \Cmc_D} \bigg(\rho_e(G) \sqcap
    \bigsqcap_{H \in \Cmc_D \setminus \{G\}} H\bigg)$,
where $\Cmc_D = N_D \cup \{D_r \mid r \in N_R\}$. 
\item Finally for $C = C_1 \sqcup
C_2 \sqcup \cdots \sqcup C_k$ we set
$  \rho_e(C) = \rho_e(C_1) \sqcup \rho_e(C_2)
    \sqcup \cdots \sqcup \rho_e(C_k)$.
    \end{itemize}
\label{def:relxaDilate}
\end{definition}

\medskip
\begin{proposition}\cite{DAB14a}
$\rho_e$ is a relaxation.
\end{proposition}
 
%
 Let us illustrate this operator \Isa{with} an example. 

\medskip 
\begin{example} 
Suppose  an agent believes that a person $\textsc{Bob}$  is married to a female judge: $T=\{\textsc{Bob}  \sqsubseteq  \textsc{male} \sqcap \exists.\textsc{MarriedTo}.\left(\textsc{female}\sqcap \textsc{judge}\right) \}$. Suppose now that due to some obscurantist law, it happens that  females are not allowed to be judges. This new belief is captured as $T'=\{\textsc{judge}\sqcap \textsc{female} \sqsubseteq \bot\}$. By applying $\rho_e$ one can resolve the conflict between the two belief sets. To ease the reading, let us rewrite the concepts as follows: $A\equiv \textsc{male}, B\equiv \textsc{female}, C\equiv \textsc{judge}, m\equiv\textsc{MarriedTo}, D \equiv \exists\textsc{MarriedTo}.\left(\textsc{female}\sqcap \textsc{judge}\right)$. Hence, from Definition~\ref{def:relxaDilate} we have $\rho_e(A\sqcap D) \equiv \left(\rho_e(A) \sqcap D\right) \sqcup \left(A \sqcap \rho_e(D)\right) $, with $\rho_e(A)\equiv \top$ and 
 \[
 \begin{split}
 \rho_e(D)\equiv & \exists m.\rho_e(B\sqcap C) \sqcup \left(\exists m.B \sqcap \exists m.\rho_e(C)\right) \sqcup  \left( \exists m.\rho_e(B)\sqcap \exists m.C\right)\\
 \equiv&\exists m.(B \sqcup C) \sqcup \left(\exists m.B \sqcap \exists m.\top \right) \sqcup \left( \exists m.\top \sqcap \exists m.C\right)\\
 \equiv   &\exists m.B \sqcup \exists m.C \sqcup \exists m.(B \sqcup C)
 \equiv \exists m.B \sqcup \exists m.C
 \end{split}
 \]
 Then 
  \[
 \begin{split}
\rho_e(A\sqcap D)\equiv & \left(\rho_e(A) \sqcap D\right) \sqcup \left(A \sqcap \rho_e(D)\right) \\
 \equiv&(\top \sqcap D) \sqcup \left(A \sqcap \left(\exists m.B \sqcup \exists m.C\right)\right)\\
 \equiv &  D  \sqcup \left(A \sqcap \left(\exists m.B \sqcup \exists m.C\right)\right)
 \end{split}
 \]
The new agent's belief, up to a rewriting, becomes\\ $\{\textsc{Bob} \sqsubseteq  \exists.\textsc{MarriedTo}.\left(\textsc{female}\sqcap \textsc{judge}\right) \sqcup \\ \left(\textsc{male} \sqcap  \left(\exists \textsc{Married}.\textsc{female} \sqcup \exists \textsc{Married}.\textsc{judge}\right)\right), \textsc{judge}\sqcap \textsc{female} \sqsubseteq \bot\}$.
 \end{example}
 
One can notice from this example that the relaxation $\rho_e$ leads to a refined revision operator. Indeed, the  resulting relaxed axiom in $T$ 
emphasizes all the minimal possible changes (through the disjunction operator) on $\textsc{Bob}$'s condition. This is due to the fact that the relaxation operator $\rho_e$ corresponds to dilating the set of models of a ball defined from an edit distance on the concept description tree of size one. For more details on the correspondence between this relaxation operator, the set of models and tree edit distances, one can refer to~\cite{DAB14a}. 

 
 Another possibility for defining a relaxation in $\ELext{}$ is  obtained by exploiting the disjunction constructor by augmenting a concept description with a set of exceptions.
 
 \medskip
 \begin{definition}[Relaxation from exceptions in $\ELext{}$]
 \label{def:rel_ex}
 Given a  set of exceptions $\mathcal{E}=\{E_1,\cdots,E_n\}$, we define a relaxation of degree $k$ of an $\ELext{}$-concept description $C$ as follows:
for a finite set $\mathcal{E}^k \subseteq \mathcal{E}$ with $|\mathcal{E}^k|=k$, \Isa{$C$ is relaxed by adding the sets $E_{i_j} \in \mathcal{E}^k$ such that $E_{i_j} \sqcap C \sqsubseteq \bot$}
 \[
 \rho^k_{\mathcal{E}}(C)=C\sqcup E_{i_1} \sqcup \cdots \sqcup E_{i_k}.
 \]
 \end{definition}
 
 \medskip
 \begin{proposition}
 $\rho^k_{\mathcal{E}}$ is extensive.
 \end{proposition}
 \begin{proof}
Extensivity  of this operator follows directly from the definition. 
\end{proof}

However, exhaustivity is not necessarily satisfied unless the exception set includes the $\top$ concept or the disjunction of some or all of its elements entails the $\top$ concept. 

{\it If we consider again Example~\ref{ex:tweety}, a relaxation of the formula $\textsc{bird}\sqsubseteq \textsc{flies}$ using the operator $ \rho^k_{\mathcal{E}}$ over the concept $\textsc{flies}$ with the exception set $\mathcal{E}=\{\textsc{Tweety}\}$ results in the formula $\textsc{bird}\sqsubseteq \textsc{flies}\sqcup \textsc{Tweety}$. The new revised knowledge base, if Condition 3 \Isa{in Definition~\ref{revision based in relaxation}} is not considered,  is then $\{\textsc{Tweety}\sqsubseteq \textsc{bird}, \textsc{bird}\sqsubseteq \textsc{flies}\sqcup \textsc{Tweety},\textsc{Tweety} \sqcap \textsc{flies} \sqsubseteq \bot\}$ which is consistent. This is obviously a more refined revision than the one obtained from the operator $\rho_{\bot}$, but requires the logic to be  equipped with the disjunction connective and the definition of a set of exceptions}.
 
 Another example involving this relaxation will be discussed in the $\ALC{}$ case (cf. Example~\ref{ex:rich}).
 
 \subsection{Relaxation and retraction in $\ALC{}$}
 We consider here operators suited to $\ALC{}$ language. Of course, all the operators defined for $\EL{}$ and $\ELext{}$ remain valid.
 
 \paragraph{$\ALC{}$-Concept Retractions.}
 A first possibility for defining retraction is to remove iteratively from an $\ALC{}$-concept description one or a set of its subconcepts. A similar construction has been introduced in~\cite{QLB06}. Interestingly enough, almost all the operators defined in~\cite{Gorogiannis2008a,QLB06} are relaxations. 
 
 \medskip
 \begin{definition}[Retraction from exceptions in$\ALC{}$]
Given a  set of exceptions $\mathcal{E}=\{E_1,\cdots,E_n\}$ and let $C$ be any $\ALC{}$-concept description. We \Isa{retract $C$ by constraining it to the elements $E_i^c$ such that $E_i\sqsubseteq C$:}
$$ \kappa^n_{\cal E}(C)=C\sqcap  E^c_1 \sqcap \cdots\sqcap  E^c_n.$$
 \end{definition}

\medskip
\begin{proposition}
$\kappa^n_{\cal E}$ is anti-extensive.
\end{proposition}
\begin{proof}
The proof follows from the definition.
\end{proof}

As for its counterpart relaxation ($\rho^k_{\mathcal{E}}$),  exhaustivity of $\kappa^n_{\cal E}$ is not necessarily satisfied unless the exception set includes the $\bot$ concept, or the conjunction of some or all of its elements entails the $\bot$ concept. 
  
{\it Consider again Example~\ref{ex:tweety}. We have $ \kappa^1_{\cal E}( \textsc{bird})= \textsc{bird}\sqcap \textsc{Tweety}^c$. The resulting revised knowledge base, if Condition 3 \Isa{in Definition~\ref{revision based in relaxation}} is not considered, is then $\{\textsc{Tweety}\sqsubseteq \textsc{bird}, \textsc{bird} \sqcap  \textsc{Tweety}^c\sqsubseteq \textsc{flies},\textsc{Tweety} \sqcap \textsc{flies} \sqsubseteq \bot\}$ which is consistent.}
 
 Another possibility, suggested in~\cite{Gorogiannis2008a} and related to operators defined in propositional logic as introduced in~\cite{BL02}, consists in applying the retraction at the atomic level. This captures somehow  Dalal's idea of revision operators in propositional logic~\cite{DALA-88}.

 \medskip
 \begin{definition}
 \label{def:dalal-er}
 Let $C$ be an $\ALC{}$-concept description of the form $Q_1r_1\cdots Q_mr_m.D$, where $Q_i$ is a quantifier and $D$ is  quantifier-free and in CNF form, i.e. $D=E_1\sqcap E_2\sqcap\cdots E_n$ with $E_{i}$ being disjunctions of possibly negated atomic concepts. Let us define, as in the propositional case~\cite{BL02}, $\kappa_p(D)=\bigsqcap_{j=1}^n(\bigsqcup_{i\neq j} E_i)$. Then 
$ \kappa_{\text{Dalal}}(C)=Q_1r_1\cdots Q_mr_m.\kappa_p(D)$.
 \end{definition}

\medskip
 \begin{proposition}
 $ \kappa^n_{\text{Dalal}}$ is a retraction.
 \end{proposition}
 \begin{proof}
 Exhaustivity and anti-extensivity follow from those of $\kappa_p$. Indeed the operator $\kappa_p$ is exhaustive and anti-extensive, and if applied $n$ times it reaches the $\bot$ concept (see~\cite{BL02} for properties of this operator). 
 \end{proof}

 
 This idea can be generalized to consider any retraction defined in $\ELext{}$.
 
 \medskip
  \begin{definition}
 Let $C$ be an $\ALC{}$-concept description of the form $Q_1r_1\cdots Q_mr_m.D$, where $Q_i$ is a quantifier and $D$ is a quantifier-free. Then \Isa{we define}
$ \kappa_{\cap}(C)=Q_1r_1\cdots Q_mr_m.\kappa^n_{\cal E}(D)$.
 \end{definition}

\medskip\begin{proposition}
$ \kappa^n_{\cap}$ is anti-extensive.
\end{proposition}
\begin{proof}
The properties  of this operator follows from the ones of $\kappa^n_{\cal E}(D)$. Hence, anti-extensivity is verified but not necessarily exhaustivity.
\end{proof}
 
 Another possible $\ALC{}$-concept description retraction is obtained by substituting  the existential restriction by an universal one. This idea has been sketched in~\cite{Gorogiannis2008a} for defining dilation operators (then by transforming $\forall$ into $\exists$), i.e. special relaxation operators enjoying additional properties~\cite{DAB14a}.  We adapt it here to define retraction in DL syntax.
 
 \medskip
   \begin{definition}
 Let $C$ be an $\ALC{}$-concept description of the form $Q_1r_1\cdots Q_nr_n.D$, where $Q_i$ is a quantifier, $D$ is quantifier-free, then \Isa{we define}
 \begin{align*}
 \kappa_{q}(C)=\bigsqcap \{Q'_1r_1\cdots Q'_nr_n.D \mid 
    \exists j\leq n \text{ s.t. } Q_j=\exists \\ \text { and } Q'_j=\forall,  \text{ and for all }  i\leq n  
 \text{ s.t. } i\neq j, Q'_i=Q_i \}
 \end{align*}
 \label{def:retract_q}
 \end{definition}

 \begin{proposition}
 $ \kappa_{q}$ is anti-extensive.
 \end{proposition}
 \begin{proof}
The proof relies on the following general result:
\[
\forall C, \forall r, \forall r.C \sqsubseteq \exists r.C
\]
Indeed, for each interpretation $\mathcal{I}$, \Isa{if $r_i^{\mathcal{I}} \neq \emptyset$,} we have
\[
x \in (\forall r.C)^\mathcal{I} \Rightarrow (\forall y, (x,y) \in r^\mathcal{I} \Rightarrow y \in C^\mathcal{I}) \Rightarrow (\exists y, (x,y) \in r^\mathcal{I} \mbox{and } y \in C^\mathcal{I}) \Rightarrow x \in (\exists r.C)^\mathcal{I} .
\]
Hence $(\forall r.C)^\mathcal{I} \subseteq (\exists r.C)^\mathcal{I}$ for each $\mathcal{I}$ (if $r_i^{\mathcal{I}} = \emptyset$ it is obvious), and $\forall r.C \sqsubseteq \exists r.C$.

In a similar way, we can show, that for any $C_1, C_2, r$, and $Q \in \{\exists, \forall\}$:
\[
C_1 \sqsubseteq C_2 \Rightarrow Qr.C_1 \sqsubseteq Qr.C_2.
\]

Now, let us consider any $j$ such that $Q_j = \exists$, and set $C'= Q_{j+1}r_{j+1} ... Q_n r_n. D$. We have from the first result $Q'_jr_j.C' \sqsubseteq Q_jr_j.C'$.
Applying the second result recursively on each $Q_i$ for $i<j$, we then have
\[
Q_1r_1... Q_{j-1}r_{j-1}Q'_jr_j.C' \sqsubseteq Q_1r_1... Q_{j-1}r_{j-1}Q_jr_j.C'.
\]
The same relation holds for the conjunction over any $j$ such that $Q_j = \exists$, from which we conclude that $\forall C, \kappa_q(C) \sqsubseteq C$, i.e. $\kappa_q$ is anti-extensive.

 \end{proof}
 
 Note that for $\kappa_{q}$ exhaustivity can be obtained by further removing recursively the remaining universal quantifiers and apply at the final step any retraction defined above on the concept $D$.

  \paragraph{$\ALC{}$-Concept Relaxations.}

 Let us now introduce some relaxation operators suited to $\ALC{}$ language.

\medskip
\begin{definition}
 Let $C$ be an $\ALC{}$-concept description of the form $Q_1r_1\cdots Q_mr_m.D$, where $Q_i$ is a quantifier and $D$ is quantifier-free and in DNF form, i.e. $D=E_1\sqcup E_2\sqcup\cdots E_n$ with $E_{i}$ being a conjunction of possibly negated atomic concepts. Define, as in the propositional case~\cite{BL02}, $\rho_p(D)=\bigsqcup_{j=1}^n(\bigsqcap_{i\neq j} E_i)$, then 
$ \rho^n_{\text{Dalal}}(C)=Q_1r_1\cdots Q_mr_m.\rho_p^n(D)$.
 \end{definition}
 
 As for retraction, this idea can be generalized to consider any relaxation defined in $\ELext{}$.

 \medskip
  \begin{definition}
 Let $C$ be an $\ALC{}$-concept description of the form $Q_1r_1\cdots Q_nr_n.D$, where $Q_i$ is a quantifier and $D$ is  quantifier-free, then \Isa{we define}
$ \rho^n_{\cup}(C)=Q_1r_1\cdots Q_nr_n.\rho^n_{\cal E}(D).$
 \end{definition}
 
  Let us consider another example adapted from the literature to illustrate these operators~\cite{QLB06}.
  
  \medskip
 \begin{example}
 \label{ex:rich}
{\it Let us consider the following knowledge bases: $T=\{\textsc{Bob}\sqsubseteq \forall \textsc{hasChild}.\textsc{rich}, \textsc{Bob}\sqsubseteq \exists  \textsc{hasChild}.\textsc{Mary}, \textsc{Mary}\sqsubseteq \textsc{rich} \}$ and $T'=\{\textsc{Bob}\sqsubseteq \textsc{hasChild}.\textsc{John}, \textsc{John}\sqsubseteq  \textsc{rich}^c\}$. Relaxing  the formula $\textsc{Bob}\sqsubseteq \forall \textsc{hasChild}.\textsc{rich}$ by applying $\rho^n_{\cup}$ to the concept on the right hand side results in the following formula  $\textsc{Bob}\sqsubseteq \forall \textsc{hasChild}.(\textsc{rich}\sqcup \textsc{John})$ which resolves the conflict between the two knowledge bases. }
 \end{example}
 
 A last possibility, dual to the retraction operator given in Definition~\ref{def:retract_q}, consists in transforming universal quantifiers to existential ones.
 
 \medskip
    \begin{definition}
 Let $C$ be an $\ALC{}$-concept description of the form $Q_1r_1\cdots Q_nr_n.D$, where $Q_i$ is a quantifier and $D$ is  quantifier-free, then \Isa{we define a relaxation as:}
 \begin{align*}
 \rho_{q}(C)=\bigsqcup \{Q'_1r_1\cdots Q'_nr_n.D \mid 
   \exists j\leq n \text{ s.t. } Q_j=\forall \\ \text{ and } Q'_j=\exists, \text{ and for all }  i\leq n  
 \text{ s.t. } i\neq j, Q'_i=Q_i \}
 \end{align*}
 \label{def:relax_q}
 \end{definition}
{\it If we consider again Example~\ref{ex:rich}, relaxing the formula $\textsc{Bob}\sqsubseteq \forall \textsc{hasChild}.\textsc{rich}$ by applying $\rho_{q}$ to the concept on the right hand side results in the following formula  $\textsc{Bob}\sqsubseteq \exists \textsc{hasChild}.\textsc{rich}$, which resolves the conflict between the two knowledge bases. }

\medskip
\begin{proposition}
The operators $\rho_{\text{Dalal}}, \rho_q$ are extensive and exhaustive. 
The operators $ \rho_{\cup}$ is extensive but not exhaustive.
\end{proposition}
\begin{proof}
The properties of $\rho_{\text{Dalal}}$ and $ \rho_{\cup}$ are directly derived from the definitions and from properties of $\rho_p$ detailed in~\cite{BL02} and $\rho_{\mathcal{E}}$.  The proof of $\rho_q$ being  extensive and exhaustive can be found in~\cite{Gorogiannis2008a}.
\end{proof}

%

\section{Related work} 
\label{related works}

Recently a first generalization of AGM revision has been proposed in the framework of Tarskian logics considering minimality criteria on removed formulas~\cite{RW14} following previous works of the same authors for contraction~\cite{RWFA13}. Representation results that make a correspondence between a large family of logics containing non-classical logics such as {\bf DL} and {\bf HCL} and AGM postulates for revision with such minimality criteria have then been obtained. Here, the proposed generalization also gives similar representation theorems (cf. Theorem~\ref{correspondence agm and orders}) but for a different minimality criterion. Indeed, we showed in Section~\ref{Orders and AGM postulates} that revision operators satisfying Postulates (G1)-(G6) are precisely the ones that accomplish an update with minimal change to the set of models of knowledge bases, generalizing to any institution the approach developed in~\cite{KM91} for the logic {\bf PL} and \cite{QY08} for {\bf DL}. However, our revision operator based on relaxation also has a minimality criterion on transformed formulas. Indeed, a simple consequence of Definition~\ref{revision based in relaxation} is the property 

\begin{center}
{\bf (Relevance)} Let $T,T' \subseteq Sen$ be two knowledge bases such that $T \circ T' = \rho^\mathcal{K}(T) \cup T'$. Then, for every $\varphi \in T$ such that $k_\varphi \neq 0$, $\rho^{\mathcal{K}'}(T) \cup T'$ is inconsistent for $\mathcal{K}' = \mathcal{K} \setminus \{k_\varphi\} \cup \{k'_\varphi = 0\}$.  
\end{center} 

This property states that only formulas that contribute to inconsistencies with $T'$ are allowed to be transformed. Our property {\bf (Relevance)} is similar to the property with the same name in~\cite{RW14,RWFA13}, but for contraction operators, and that states that only the formulas that somehow ``contribute'' to derive the formulas to abandon can be removed.





\Jamal{Since the primary aim of this paper is to show that a more general framework, encompassing different logics, can be useful, it is out of the scope of this paper to provide an overview of all existing relaxation methods. However, some works deserve to be mentioned, since they are based on ideas that show some similarity with the relaxation notion proposed in our framework.}

\Jamal{The relaxation idea originates from the work on Morpho-Logics, initially introduced in~\cite{BL02,BPU04}. In this seminal work, revision operators (and explanatory relations) were defined through dilation and erosion operators. These operators share some similarities with relaxation and retraction as defined in this paper. Dilation is a sup-preserving operator and erosion is inf-preserving, hence both are increasing. Some particular dilations and erosions are exhaustive and extensive while relaxation and retraction operators are defined to be exhaustive and extensive but not necessarily sup- and inf-preserving. Dilation  has been further exploited for merging first-order theories~in~\cite{Gorogiannis2008a}.}

\Jamal{In~\cite{AGM85}, the notion of partial meet contraction is defined as the intersection of a non empty family of maximal subsets of the theory than do not imply the proposition to be eliminated. Revision is then defined from the Levi identity. The maximal subsets can also be selected according to some choice function. The authors also define a notion of partial meet revision, which can be seen as a special case of the relaxation operator introduced in this paper. 
In~\cite{hansson1996}, the author also discusses choice functions and compares the postulates for partial meet revision to the AGM postulates. He also highlights the distinction between belief sets (which can be very large) and belief bases (which are not necessarily closed by $Cn$). More precisely, $A$ is a belief base of a belief set $K$ iff $K = Cn(A)$. A permissive belief revision is defined in~\cite{cravo2001}, based on the notion of weakening. The beliefs which are suppressed by classical revision methods are replaced by weaker forms, which keep the resulting belief set consistent. This notion of weakening is closed to the one of relaxation developed in this paper.
In the last decade, several works have studied revision operators in description logics. While most of them concentrated on the adaptation of AGM theory, few works have addressed  the definition of concrete operators~\cite{meyer2005knowledge,qi2006revision,QLB06,qi2009}. For instance, in~\cite{meyer2005knowledge}, based on the seminal work in~\cite{benferhat2004}, revision in DL is studied by defining strategies to manage inconsistencies and using the notion of knowledge integration (see also the work by Hansson). The authors propose a conjunctive maxi-adjustment, for stratified knowledge bases and lexicographic entailment. In~\cite{qi2006revision}, weakening operators, that are in fact relaxation operators, are defined. Our work brings a principled formal flavor to these operators. In~\cite{qi2009}, revision of ontologies in DL is based on the notion of forgetting, which is also a way to manage inconsistencies. The authors propose a model based approach, inspired by Dalal's revision in {\bf PL}, and based on a distance between terminologies and on the difference set between two interpretation. The models of the revision $T \circ T'$ are then the interpretations $\mathcal{I}$ for which there exists an interpretation $\mathcal{I'}$ such that the cardinality of the difference set between $\mathcal{I}$ and $\mathcal{I'}$ is equal to the distance between $T$ and $T'$. In~\cite{liu2006}, updating Aboxes in DL is discussed, and some operators are introduced. The rationality of these operators is not discussed, hence the interest of a formal theory such as the AGM postulates.  In~\cite{autexier2013} an original use of  DL revision is introduced for  the orchestration of processes.  A closely related field is inconsistency handling in ontologies (e.g.~\cite{schlobach2003non,schlobach2007debugging}), with the main difference that the rationality of inconsistency  repairing operators is not investigated, as suggested by the AGM theory.}

\Jamal{As said before, some of our DL-based relaxation operators are closely related to the ones introduced in~\cite{QLB06} for knowledge bases revision. Our relaxation-based revision framework, being abstract enough (i.e. defined through easily satisfied properties), encompasses these operators.  Moreover,  the revision operator defined in~\cite{QLB06} considers only  inconsistencies due to Abox assertions. Our operators are general in the sense that Abox assertions are handled as any formula of the language.}


\section{Conclusion}

\Jamal{The contribution of this paper is threefold. First, we
provided a generalization of AGM postulates from a model-theoretic point of view, by defining this operator in an abstract model theory known under the
name of satisfaction systems, so as they become applicable to a wide class of non-classical logics. In this framework,
we then generalized to any satisfaction systems the characterization
of the AGM postulates given by Katsuno and Mendelzon
for propositional logic in terms of minimal change
with respect to an ordering among interpretations. This work generalizes the previous ones in the area. It also suggests the theory behind  satisfaction systems to be a principled framework for dealing with knowledge dynamics with the growing interest on non-classical logics such as DL.  We do hope that bridges can thus be built, by working at the cross-road of different areas of theoretical computer sciences.}

\Jamal{Secondly, we proposed a general framework for defining revision operators based on the notion of relaxation. We demonstrated that such a relaxation-based framework for belief revision satisfies the AGM postulates. As a byproduct, we give a principled formal flavor to several operators defined in the literature (e.g. weakening operators defined in DL).}

\Jamal{Thirdly, we introduced a bunch of concrete relaxations within the scope of description logics, discussed their properties and illustrated them through simple examples. It was out of the scope of this paper to discuss tools such as OWL. However, the proposed approach could be applied to SROIQ and implemented in OWL, by augmenting a relaxation  with operations on complex constructors.} 


\Jamal{Future works will concern the study  of  the complexity of the introduced operators,  the comparison of their induced ordering, and their  generalization to more expressive DL as well as other non-classical logics such as first-order Horn logics or equational logics.}






\bibliographystyle{elsart-num-sort}
\bibliography{biblio,DL_revision_MM}

\end{document}